\documentclass[10pt]{article}
\usepackage{amssymb,amsmath,amsfonts,amsthm,mathtools,color}

\usepackage{mathrsfs}
\usepackage{bm} 

\usepackage{comment} 
\usepackage[title,titletoc,header]{appendix}
\usepackage{graphicx,subfig}

\usepackage{paralist}

\usepackage{tikz}
\usetikzlibrary{arrows,positioning,shapes.geometric}

\usepackage[linesnumbered, ruled,vlined]{algorithm2e}
\SetKwInput{KwInput}{Input}                
\usepackage{soul}

\graphicspath{ {./figures/} }

\usepackage{indentfirst}
\usepackage{multicol}
\usepackage{booktabs}
\usepackage{url}
\usepackage[outdir=./]{epstopdf} 

\usepackage[shortlabels]{enumitem}

\usepackage{hyperref}
\hypersetup{
    colorlinks=true, 
    linktoc=all,     
    linkcolor=blue,  
}
\numberwithin{equation}{section}

\newtheorem{Theorem}{Theorem}[section]
\newtheorem{Lemma}[Theorem]{Lemma}
\newtheorem{Proposition}[Theorem]{Proposition}
\newtheorem{Assumption}{H.\!\!}

\theoremstyle{definition}
\newtheorem{Definition}{Definition}[section]
\newtheorem{Setting}{Setting}[section]
\newtheorem{Example}{Example}[section]
\newtheorem{Approach}{Approach}

\theoremstyle{remark}
\newtheorem{Remark}{Remark}[section]

  \def\nb{\nonumber}
\def\to{\rightarrow}
 \def\ol{\overline}    \def\ul{\underline}
\def\Om{\Omega}  \def\om{\omega} 
 
\newcommand{\q}{\quad}

  \def\fa{\forall}

\def\eps{\varepsilon}

 \def\t{\times}  
\def\ms{\medskip}

\def\cB{\mathcal{B}}

\def\cF{\mathcal{F}}
\def\cG{\mathcal{G}}
\def\cH{\mathcal{H}}

\def\cM{\mathcal{M}}
\def\cN{\mathcal{N}}
\def\cO{\mathcal{O}}
\def\cP{\mathcal{P}}

\def\cR{\mathcal{R}}
\def\cS{\mathcal{S}}

\def\cX{\mathcal{X}}
\def\cY{\mathcal{Y}}
\def\cZ{\mathcal{Z}}

\def\d{{\mathrm{d}}}

\def\sE{{\mathbb{E}}}

\def\sN{{\mathbb{N}}}
\def\sP{\mathbb{P}}

\def\sR{{\mathbb R}}
\def\sS{{\mathbb{S}}}
\def\sW{{\mathbb{W}}}
\def\sX{{\mathbb{X}}}
\def\sY{{\mathbb{Y}}}

\DeclareMathOperator*{\argmin}{arg\,min}

\DeclareMathOperator*{\esssup}{ess\,sup}
\DeclareMathOperator*{\essinf}{ess\,inf}

\newcommand{\lc}
{\mathrel{\raise2pt\hbox{${\mathop<\limits_{\raise1pt\hbox
{\mbox{$\sim$}}}}$}}}

\newcommand{\gc}
{\mathrel{\raise2pt\hbox{${\mathop>\limits_{\raise1pt\hbox{\mbox{$\sim$}}}}$}}}

\newcommand{\ec}
{\mathrel{\raise2pt\hbox{${\mathop=\limits_{\raise1pt\hbox{\mbox{$\sim$}}}}$}}}

\def\bb{\begin{equation}} \def\ee{\end{equation}}
\def\bbn{\begin{equation*}} \def\een{\end{equation*}}

\def\beqn{\begin{eqnarray}}  \def\eqn{\end{eqnarray}}

\def\beqnx{\begin{eqnarray*}} \def\eqnx{\end{eqnarray*}}

\def\bn{\begin{enumerate}} \def\en{\end{enumerate}}

\def\bd{\begin{description}} \def\ed{\end{description}}



\begin{document}

\title{Optimal scheduling of entropy regulariser
for continuous-time 
linear-quadratic reinforcement learning
}

\author{
\and 
 Lukasz Szpruch\thanks{School of Mathematics, University of Edinburgh and Alan Turing Institute,  \texttt{L.Szpruch@ed.ac.uk}}
\and Tanut Treetanthiploet\thanks{Alan Turing Institute,  \texttt{ttreetanthiploet@turing.ac.uk}}
\and Yufei Zhang\thanks{Department of Statistics, London School of Economics and Political Science,  \texttt{y.zhang389@lse.ac.uk}}
}
\date{}

\maketitle

\noindent\textbf{Abstract.} 
This work uses the entropy-regularised relaxed stochastic control perspective as a principled framework for designing
reinforcement learning (RL) algorithms. Herein agent interacts with the environment by generating noisy controls  distributed according to the optimal relaxed policy.
The noisy policies 
on the one hand, explore the space and hence facilitate learning but, on the other hand, introduce bias by assigning a positive probability to non-optimal actions. 
This exploration-exploitation trade-off is determined by the strength of entropy regularisation.
We study  algorithms resulting
from two  entropy regularisation formulations:
the exploratory control approach, where entropy is added to the cost objective,  and the proximal policy update approach, where entropy penalises policy divergence 
between   consecutive episodes. 
We focus on  the finite horizon continuous-time linear-quadratic (LQ) RL problem,
{where a  linear dynamics with unknown drift coefficients is controlled subject to  quadratic costs.} In this setting, 
both algorithms yield a Gaussian relaxed policy. 
We quantify the precise difference between the value functions of a  Gaussian policy and its noisy evaluation
and show that the execution noise 
must be independent across time. 
By tuning the frequency of sampling from  relaxed policies and
the parameter governing the strength of entropy regularisation, we prove that the regret, for both learning algorithms,  
is of the order $\cO(\sqrt{N}) $ 
(up to a  logarithmic factor)
 over $N$ episodes, matching the best known result from the literature.

\medskip
\noindent
\textbf{Key words.} 
Continuous-time
reinforcement learning,
linear-quadratic,
 entropy regularisation,
exploratory control,
proximal policy update,
regret analysis

%
\ms
\noindent
\textbf{AMS subject classifications.} 
62L05,  49N10, 93E35, 94A17




%
%

\medskip

\section{Introduction}
\label{sec:introduction}

Reinforcement learning (RL) is concerned with sequential decision-making in an uncertain environment and is a core topic in machine learning.
Two key concepts in RL are \textit{exploration}, which corresponds to learning via interactions with the 
 environment,  and \textit{exploitation}, which corresponds to optimising the objective function given accumulated information (see \cite{sutton1998introduction}). 
They are    at odds with each other,
as learning the environment   often  requires   acting suboptimally  
with respect to existing knowledge \cite{keskin2018incomplete}. 
%
Thus it is crucial to develop 
effective exploration strategies 
and to optimally  balance exploration and exploitation.

\paragraph{Randomisation   and entropy-regularised relaxed control.}
 A common exploration technique in RL is
   to adopt     randomised actions.
   Randomisation generates 
stochastic policies that   explore the  environment, but simultaneously  introduces   additional biases by assigning positive probability to, potentially, non-optimal actions. 
   The precise impact of randomisation     has been studied extensively for
  discrete-time  RL problems, including  finite state-action models   \cite{even2006action, cesa2017boltzmann}
and 
   linear-quadratic (LQ) models
\cite{ 
mania2019certainty,dean2020sample,simchowitz2020naive}.
This is often quantified by the      \emph{regret} of learning,
which      measures  the  difference
 between  the value function  of   designed policies and the optimal  value function that could have been achieved  had  the agent   known the environment.
 It is shown that to achieve the optimal regret, 
 the level of   randomisation  
 has to decrease at a proper rate   over the learning process. 
 
For RL problems in continuous time and space, 
\cite{wang2020reinforcement}
proposes
an entropy-regularised   relaxed control framework
to systematically  design stochastic polices. Relaxed control formulation  
has been introduced  in classical control theory
for the existence of optimal controls 
(see e.g., \cite{young2000lectures,nicole1987compactification}), 
but 
its entropy regularised variants
   have recently attracted much attention
for    their wide applications in learning.
 Entropy regularisation ensures the  optimal stochastic policy 
has a nondegenerate distribution over actions 
and hence 
explores the environment
(see e.g.,
\cite{wang2020reinforcement,wang2020continuous,
cohen2020asymptotic,
jia2021policy,
firoozi2022exploratory, guo2022entropy, jaimungal2022reinforcement}). 
It also
 leads to  a  policy    that is   Lipschitz stable  
 with respect to perturbations of the underlying model
 \cite{reisinger2021regularity}.
Such a    stability property  is   crucial for the  analysis of 
algorithm regret  
\cite{guo2021reinforcement,szpruch2021exploration}.
These advantages 
  make  entropy-regularised  relaxed stochastic control formulation a good choice for the principled design of efficient learning algorithms.

However, despite the recent influx of interest on 
entropy-regularised control problems, to the best of our knowledge 
there is no work on how to control 
the  entropy regularising weight
  to 
achieve the best possible  algorithm  regret. 
A large degree of  entropy regularisation often signifies a high level of exploration,
which    facilitates more efficient learning 
but also  creates a larger  exploration cost. Furthermore relaxed control formulation does not constitute a complete learning algorithm 
as the set of admissible policies for the agent are 
functions mapping
time and space into actions and not 
into
 probability distributions over all actions.

Therefore, finding the  optimal scheduling  of entropy regularisation requires
addressing   the following two questions:
A) How to follow a relaxed policy to interact with 
a continuous-time system?\footnotemark
\footnotetext{The terminologies ``relaxed policy" and ``stochastic policy" will be used interchangeably  
throughout the paper.}
 B) How do the policy execution and regularising weight  
 affect the exploration-exploitation trade-off over the entire learning process?

\paragraph{Main Contributions.}
This works addresses the above  questions for 
 finite  horizon continuous-time episodic LQ-RL problems. 
 {In this problem, the environment is modelled by a linear stochastic differential equation with   unknown drift parameter and   known deterministic diffusion coefficient,
and  the agent is   minimising   known quadratic costs over a finite   horizon.

 } 
  
\begin{itemize}
\item We prove that a Gaussian relaxed policy 
can be approximated with a single episode
by 
piecewise constant independent Gaussian noises    on a time grid.
The type of execution noise  
corresponds to a  properly scaled discrete-time  white noise process 
(see Section \ref{sec:discussion} for details). 
The error of this noisy execution 
is of half-order in both the   mesh size of the grid  and  
the variance of the relaxed   policy
in the high-probability sense,
and 
is 
 of first-order in both the   mesh size 
and  the variance  
under expectation 
 (Theorem \ref{Thm:execution gap_drift_control}). 
This result should be contrasted with the heuristic implementation
of a relaxed policy 
 suggested in \cite{wang2020reinforcement,firoozi2022exploratory}, { which involves independent agents interacting over 
a large number of episodes
(see Example \ref{example:repetition}).
 }



\item 
We use entropy-regularised relaxed control problem formulation
to  design   algorithms that achieve the best-known regret 
for episodic LQ-RL problems.
 At each episode, the algorithm
 executes a Gaussian relaxed policy on a time grid, 
estimates the unknown parameters by a 
Bayesian approach 
 based on observed trajectories,
 and then 
 designs a Gaussian policy by solving 
 an entropy-regularised control problem. 
Two widely used   entropy regularisation formulations 
are analysed: 
the exploratory control approach 
(i.e., Algorithm \ref{alg:exploration})
where    entropy serves as an  exploration reward
\cite{wang2020reinforcement,wang2020continuous, 
firoozi2022exploratory, guo2022entropy, jaimungal2022reinforcement},
and the proximal policy update approach 
(i.e., Algorithm \ref{alg:MD})
where entropy  penalises the divergence of policies 
between two consecutive episodes
\cite{schulman2017proximal,geist2019theory, kerimkulov2021modified}.
By optimising the execution mesh size and regularising weight, 
both algorithms achieve 
an expected  regret of the order 
$\cO(\sqrt{N})$ (up to a  logarithmic factor) over $N$ episodes,
 matching the best possible results from the literature
(see Theorems \ref{thm:vanish_exploration} and \ref{thm:MD}).
 

\end{itemize} 
 
 \paragraph{Related works.}
 To the best of our knowledge, this is the first theoretical work on
 the execution of a relaxed policy, 
on the    optimal scheduling of entropy regulariser  for RL problems,  
and on regret bounds of learning algorithms with proximal policy updates. 

%

{Optimal control of stochastic systems with parametric uncertainty has been studied in the classical adaptive control literature, where the goal is to construct a stationary policy that minimises the long-term average cost \cite{sastry2011adaptive}. 
For LQ adaptive control problems, 
the   convergence of adaptive control algorithms  has been shown in \cite{duncan1999adaptive} as the time   goes to infinity,
and   {non-asymptotic}   regret bounds   have been established  in   
\cite{abbasi2011regret, cohen2019learning, mania2019certainty, dean2020sample, simchowitz2020naive, faradonbeh2022thompson}.
The problem studied here is different. Our main objective is to construct optimal  time-dependent  policies for finite-horizon problems, 
and to 
establish a principled framework for designing RL   algorithms using the perspective of entropy-regularised stochastic control.
 Compared to the   adaptive control literature, the regret analysis in this work
 requires novel techniques,  including precise performance 
 estimation  of a randomised policy 
  and  quantification of the impact of   the   regularising weight 
  on the    regret  bound   for  continuous-time systems. 
 
}

Available   theoretical results on regret bounds for continuous-time RL are very limited, including 
\cite{gao2022logarithmic,gao2022square}
for  continuous-time  Markov decision processes (with finite states and actions),  
\cite{faradonbeh2022thompson}
 for infinite-horizon LQ-RL problems  with stationary polices, 
  \cite{guo2021reinforcement,basei2022logarithmic, szpruch2021exploration}
  for   finite-horizon  problems with time dependent policies. 
%
%
{Specifically, for RL problems with continuous-time LQ models, \cite{basei2022logarithmic} shows that when the agent knows a priori that the optimal control of the true model exploits the parameter space, a pure exploitation algorithm based on optimal policies of the current estimated models leads to a logarithmic regret. In the more general setting where the optimal control of the true model does not guarantee exploration, \cite{szpruch2021exploration} proposes a phased-based algorithm that explicitly separates exploitation and exploration phases and achieves a square-root regret. The algorithm therein uses a fixed deterministic exploration policy for exploration throughout the learning process and (deterministic) optimal policies of the current estimated models for exploitation.
}

The regret analysis of randomised policies considered in this paper is more  challenging than  
 those for deterministic policies. Since the randomised policies explore and exploit simultaneously, one needs to carefully disentangle the impacts of the injected noises on exploitation and exploration in order to recover the square-root regret bounds.    
 Moreover, 
the   proximal policy update approach
  (i.e., Algorithm \ref{alg:MD})
results in  randomised policies   depending on 
\emph{all previously estimated models}.  
Such a     memory dependence  has to be carefully controlled via
the regularising weights to optimise the regret. 

{\paragraph{Organisation of the paper:} 
Section \ref{sec:main} formulates the LQ-RL problem with stochastic policies and presents two learning algorithms with their regret bounds. 
In particular, 
Section \ref{sec:relaxed_control_problem} introduces two entropy regularisation formulations for designing learning algorithms. It is proven that both formulations yield Gaussian policies, providing a principled approach to algorithm design. 
Section \ref{sec:main_executation} analyses sampling procedures to  execute  a Gaussian policy, and establishes bounds on the execution error   in terms of the sampling frequency.
Section \ref{sec:regret_analysis} optimises the policy sampling frequency and entropy regularisation weight to achieve square-root regrets   for both algorithms.  
The proofs of the execution error bounds and algorithm regrets are presented in Sections \ref{sec:execution_general} and \ref{sec:regret_proof}, respectively.

 } 
   \paragraph{Notation:} 

    For each  $n\in \sN$, we denote by 
$\sS^n$  
(resp.~$\sS^n_+$)
 the space of $n\t n$ symmetric, 
(resp.~symmetric positive definite)  matrices,
and    by 
  $\Lambda_{\min}(A)$    the  smallest eigenvalues of  a given matrix $A\in \sS^n$.
  For each 
$T>0$,
 complete probability space
$(\Om,\cF, \sP)$
 and
  Euclidean space $(E,|\cdot|)$,
we denote by  
$L^\infty([0,T]; E)$
  (resp.~$C([0,T];E)$)
  the space of  
 measurable 
(resp.~continuous) functions $\phi:[0,T]\to E$ satisfying 
    $\|\phi\|_{\infty}=\esssup_{t\in[0,T]} |\phi_t|<\infty$,
    and by 
  $L^\infty(\Om\t [0,T];E)$ 
 the space of  
  measurable functions $\phi:\Om \t [0,T]\to E$ satisfying 
    $\|\phi\|_{L^\infty}=\esssup_{(\om, t)\in \Om \t [0,T]} |\phi_t(\om)|<\infty$.
For notational simplicity, we write
 $\sR_{> 0}=(0,\infty)$, 
 $\sR_{\ge 0}=[0,\infty)$, 
 and denote
 by
$\cN(m,s^2)$,
$m\in \sR$ and $s \ge 0$,
 the 
Gaussian measure on $\sR$
with mean $m$
and standard deviation  $s$.

\section{Problem formulation and main results}
\label{sec:main}

This section studies a LQ-RL problem,
  where  the state dynamics 
  involves 
  unknown   drift coefficients.
We prove that 
executing
optimal relaxed policies 
of suitable 
entropy-regularised control problems leads to  learning  algorithms
with optimal   regrets.
For the clarity of presentation, 
we assume all variables are one-dimensional, 
but the analysis and results can be naturally extended  to a multidimensional setting.

\subsection{LQ-RL with stochastic policies}
\label{sec:lq_rl_outline}

\paragraph{LQ control.} 
Let us  
first recall  the standard LQ control problem over feedback controls.
Let $T>0$,
$(\Om,\cF,\sP)$  be a complete probability space
on which a   Brownian motion $W=(W_t)_{t\in [0,T]}$
is defined,
and $\cF_0\subset\cF$ be a $\sigma$-algebra independent of $W$.
For fixed  ${ \theta^\star}=(A^\star,B^\star)\in \sR^{1\t 2}$, consider 
 minimising the following cost functional:
\begin{equation}
\label{eq:LQ_RL_cost}
     J^{  \theta^\star}(\phi)
     \coloneqq\sE\Bigg[
    \int_0^T  f\big(t,X^{{  \theta^\star},\phi}_t,\phi(t,X^{{ \theta^\star},\phi}_t)\big)\,\d t
    +g(X^{{ \theta^\star},\phi}_T)
    \bigg],
\end{equation}
over all measurable functions $\phi:[0,T]\t \sR\to \sR$
such that 
 the  controlled dynamics 
\begin{equation}
    \label{eq:LQ_RL_state}
    \d X_t = 
    (A^\star X_t +B^\star \phi(t,X_t))\, \d t
     +  
    \bar{\sigma}_t\,
     \d W_t,
     \quad  t\in [0,T]; \quad X_0 = x_0
\end{equation}
admits a unique square integrable   solution $X^{{ \theta^\star},\phi}$
 satisfying $\sE\left[\int_0^T
 |\phi(t,X^{{ \theta^\star},\phi}_t)|^2\,\d t\right]<\infty$,
and $f$ and $g$ satisfy
for all $(t,x,a)\in [0,T]\t \sR\t \sR$,
 \bb\label{eq:learning_cost}
 f(t,x,a)= Q_tx^2+2 S_txa+R_ta^2+2 p_tx+2 q_t a,
        \quad 
        g(x)=Mx^2+2 {m} x.
 \ee
We impose the following conditions on the  the  coefficients.

  \begin{Assumption}
  \label{assum:lq_rl}
 $T>0$, $x_0,A^\star,B^\star \in \sR $,
   $\bar{\sigma}, 
Q,S, R, p,q \in C([0,T];\sR)$,
$M\ge 0$, ${m}\in \sR$,  
 and 
for all $t\in [0,T]$,
$\bar{\sigma}_t, R_t>0 $ and   $Q_t-S^2_tR^{-1}_t\ge 0$.
  \end{Assumption}


Given $\theta^\star$,
under   (H.\ref{assum:lq_rl}),
standard LQ control theory
(see e.g., \cite
{yong1999stochastic})
shows that 
\eqref{eq:LQ_RL_cost}-\eqref{eq:LQ_RL_state}
admits the optimal policy $\phi^{ \theta^\star}$,
and 
for all $\theta=(A,B)\in \sR^{1\t 2}$
and $(t,x)\in [0,T]\t \sR$
\begin{equation}
\label{eq:K_k_theta}
 \phi^{ \theta}(t,x)= k^{ \theta}_t+K_t^{ \theta} x,
 \q \textnormal{with
$K_t^{ \theta}\coloneqq
-(BP^{ \theta}_t+S_t)R^{-1}_t,
\quad k_t^{\theta}\coloneqq
-(B\eta^{ \theta}_t+q_t)R^{-1}_t$}, 
\end{equation} 
where $P^{ \theta}\in C([0,T]; \sR_{\ge 0})$ and $\eta^{ \theta}\in C([0,T]; \sR)$
satisfy    for all $t\in [0,T]$,
\begin{subequations}
\label{eq:riccati_theta}
\begin{alignat}{2}
  &\tfrac{\d }{\d t} P^{ \theta}_t
  +2A P^{ \theta}_t-(B P^{ \theta}_t+S_t)^2 R_t^{-1}+Q_t=0,
&&\quad
  P^{ \theta}_T=M;
\\
 &\tfrac{\d }{\d t}  \eta^{ \theta}_t
  +\left( A-
  (B P^{ \theta}_t+S_t)R_t^{-1}B
 \right)\eta^{ \theta}_t
 + p_t- ( B P^{ \theta}_t+S_t) R^{-1}_t q_t=0, 
 &&\quad 
 \eta^{\theta}_T={m}.
\end{alignat}
\end{subequations}
 

  \paragraph{Episodic LQ-RL  and    stochastic policy.}
  
  Now assume that  the parameter $\theta^\star= (A^\star, B^\star)$ 
is unknown to the agent. Then
the LQ control problem becomes an LQ-RL problem: the agent's  objective is to control the system optimally while simultaneously learning the parameter
$\theta^\star$.
In an
episodic RL   framework, 
 agent
learns about ${ \theta^\star}$ by executing a sequence of  policies  and 
observing 
realisations of the corresponding controlled dynamics
(see Section \ref{sec:regret_analysis} for details). 
For simplicity, we assume that the agent knows  the  form of the dynamics \eqref{eq:LQ_RL_state} 
(except for the coefficient $\theta^\star$),
the diffusion coefficient $\bar{\sigma}$, 
 the cost functions $f$ and $g$ in  \eqref{eq:learning_cost} and the time horizon.
The agent is unaware of the exact value of $\theta^\star$ but possesses prior knowledge of an upper bound on $\theta^\star$ as stated in   (H.\ref{eq:bounded}).

Note  that 
learning the dynamics
often requires 
  explicit exploration.
  Indeed, as shown in \cite{szpruch2021exploration,basei2022logarithmic},
a greedy policy $\phi^\theta$, as defined in \eqref{eq:K_k_theta}, based on 
the present estimate $\theta$ of $\theta^\star$
  in general does not guarantee exploration and consequently fails to converge to the optimal solution.
As alluded to in Section \ref{sec:introduction},
  a common strategy in RL 
  is   first constructing  a 
stochastic policy $\nu: [0,T]\t \sR\to \cP(\sR)$,
which maps 
the current state to a probability measure over the action space, 
based on the current estimated model,  
and then    interacting with the environment by 
   sampling noisy actions according to the policy $\nu$.
 
{Here we outline  the episodic  learning procedure with stochastic policies.
 At the $m$-th episode, the agent designs a stochastic policy $\nu_m$, interacts with 
\eqref{eq:LQ_RL_state} by sampling from $\nu_m$,
and observes a trajectory of the corresponding   state dynamics.
Sampling actions from  a stochastic policy is referred to as its \emph{execution}, 
and the resulting randomised policy is 
denoted by $\varphi^m$.\footnotemark
\footnotetext{Note that a random policy   maps the current state and time to the original action space, 
but involves   additional randomness due to the sampling process. 
We use $\varphi$ to  distinguish it  with the deterministic policy $\phi$.}
The agent then 
 constructs a new  stochastic policy $\nu_{m+1}$ based on all previous observations,
 and proceeds to the next learning episode. 
The mechanism for sequentially updating and executing $(\nu_m)_{m\in \mathbb{N}}$ is referred to as a learning algorithm.
A agent aims to design a learning algorithm 
that minimizes the growth rate of cumulative loss  as the number of interactions tends to infinity. 
The precise formulation of this framework, accounting for different sources of randomness in state observation, policy construction, execution, and evaluation, will be discussed in detail in Section \ref{sec:regret_analysis}.
 }

\subsection{Policy construction via regularised relaxed control}
\label{sec:relaxed_control_problem}

A principle approach to construct stochastic policies is to solve entropy-regularised relaxed control problems
(see e.g., \cite{wang2020reinforcement,wang2020continuous, 
jia2021policy,
firoozi2022exploratory, guo2022entropy}).
Here we fix
 $\theta=(A,B)$  as  the current estimate of $\theta^\star$ after the $m$-th episode,
 and 
 present two relaxed control  formulations  that are widely used 
in the literature.
Let    
 $\varrho>0$  be a regularising weight,  and for each 
 measure  $\bar{\nu}$ on $\sR$, 
let $\cH(\cdot \Vert \bar{\nu}):
\cP(\sR)\to \sR\cup\{\infty\}$ 
be the relative  entropy 
with respect to $\bar{\nu}$ such that  
 for all $\nu \in \cP(\sR)$,
 \bb
 \label{eq:KL_divergence}
  \cH(\nu \Vert \bar{\nu})
=\begin{cases}
\int_\sR \ln \big(\frac{\nu(\d a)}{\bar{\nu}(\d a)}\big)\,\nu(\d a), 
& \textnormal{$ \nu$   is absolutely continuous with respect to $\bar{\nu}$,}
\\
\infty,
& \textnormal{otherwise.}
\end{cases}
 \ee

\begin{Approach}[Exploration reward]
\label{ex:exploration}
A stochastic   
policy
can be chosen 
as  the optimal   policy of 
an exploratory control problem 
 (see e.g., \cite{wang2020reinforcement,wang2020continuous,
reisinger2021regularity, 
firoozi2022exploratory, guo2022entropy, jaimungal2022reinforcement}): 
consider 
minimising 
\begin{align}
\label{eq:cost_relax_entropy}
    \widetilde{J}^\theta_{\varrho}(\nu
    )=
    \sE\Bigg[
    \int_0^T 
    \left(\int_\sR 
    f(t,X^{\theta, \nu}_t,a)\nu(t, X^{\theta, \nu}_t; \d a)
    +\varrho \cH\big(\nu(t, X^{\theta, \nu}_t)\Vert \nu_{\rm Leb}\big)
    \right)
    \,\d t
    +g(X^{\theta, \nu}_T)
    \bigg]
\end{align}
over all $\nu\in \cM$,
where $\cM$ consists of all 
measurable functions $\nu:[0,T]\t \sR\to \cP(\sR)$ such that 
 the following controlled dynamics
 \begin{equation}
\label{eq:state_relax_entropy}
    \d X_t = 
    \int_\sR (AX_t+Ba) \nu(t,X_t; \d a)\, \d t
    +  
    \bar{\sigma}_t\,
     \d W_t,
     \quad t\in [0,T]; \quad X_0 = x_0
\end{equation}
admits a unique square integrable   solution $X^{{ \theta},\nu}$
 satisfying $\sE\left[\int_0^T
 \int_\sR |a|^2
 \nu(t,X^{ \theta,\nu}_t;\d a)\,\d t\right]<\infty$,
  $\bar{\sigma}$, $f$ and $g$
are the same functions as   in \eqref{eq:LQ_RL_cost}-\eqref{eq:LQ_RL_state},
and $\nu_{\rm Leb}$ is the Lebesgue measure on $\sR$.

The  term $\varrho \cH(\cdot\Vert \bar{\nu}_{\rm Leb})$  
is an additional   reward to encourage exploration. 
Under (H.\ref{assum:lq_rl}),
standard verification arguments 
(see e.g., \cite{wang2020reinforcement,wang2020continuous}) show
that the optimal policy of \eqref{eq:cost_relax_entropy}-\eqref{eq:state_relax_entropy}
is  Gaussian:
\begin{equation}
    \label{eq:optimal_relax_theta_rho}
\nu^\theta_\varrho(t,x)=
\cN\left(  k^\theta_t+K_t^\theta
x,
\frac{\varrho}{2R_t}
\right),
\quad   (t,x)\in [0,T]\t \sR,
\end{equation}
where $K^\theta$ and $k^\theta$
are defined as in 
\eqref{eq:K_k_theta}.
Note that $\varrho>0$ ensures that 
 $\nu^\theta_{\varrho}(t,x)$ 
executes each action with  positive probability
and hence explores the parameter space.


\end{Approach}

\begin{Approach}[
Proximal
policy update]
\label{ex:MD}
A stochastic    
policy   can also be determined   by 
 penalising the divergence of policies between two consecutive episodes.
Recall that 
the greedy policy $\phi^{\theta}$ 
in \eqref{eq:K_k_theta} satisfies the following pointwise minimisation condition (see e.g., \cite[p.~317]{yong1999stochastic}): 
 \bb\label{eq:hamiltonian_greedy}
\phi^{\theta}(t,x)=\argmin_{a\in \sR} 
 H^\theta\left(
 t,x,a,2(P^{\theta}_t x+\eta^{\theta}_t)
 \right),
 \quad    (t,x)\in [0,T]\t \sR,
\ee
where $P^{\theta}$ and $\eta^{\theta}$ are  defined as in
\eqref{eq:riccati_theta},
and    $H^\theta:[0,T]\t \sR\t \sR\t \sR\to \sR$
is the Hamiltonian 
of \eqref{eq:LQ_RL_cost}-\eqref{eq:LQ_RL_state}
with $\theta^\star=\theta$:
$$
H^\theta(t,x,a,y)\coloneqq (Ax+B a)y + f(t,x,a),
\quad   (t,x,a,y)\in [0,T]\t \sR\t \sR\t \sR.
$$
Instead of taking the greedy update 
\eqref{eq:hamiltonian_greedy}, 
the agent  
determines a stochastic policy
(for the $(m+1)$-th episode)
by
minimising the following entropy-regularised Hamiltonian:
for each $(t,x)\in [0,T]\t \sR$, 
 \bb
\label{eq:minimise_hamiltonian_relax}
{\nu}^{\theta}_\varrho (t,x)=\argmin_{\nu\in \cP(\sR)} 
\left(
 \int_{\sR }
 H^\theta\left(
 t,x,a,2(P^{\theta}_t x+\eta^{\theta}_t)
 \right)
  \, \nu(\d a)
  +\varrho \cH
  \left(\nu\Vert {\nu}^{m}(t,x)\right)
\right),
\ee
where
   $\nu^{m}$ 
 is  the    relaxed policy for  the $m$-th episode. 

The relative entropy in  \eqref{eq:minimise_hamiltonian_relax}
enforces the updated policy ${\nu}^{\theta}_\varrho(t,x)$ to stay close 
to ${\nu}^{m}(t,x)$.
This  prevents  an excessively large policy
update,
and ensures that the updated policy ${\nu}^{\theta}_\varrho$
explores the space 
if $\nu^m$ is a relaxed policy.
Similar ideas  have been used to design proximal policy gradient methods in \cite{schulman2017proximal,kerimkulov2021modified} and 
modified value iterations in \cite{geist2019theory}.
The update rule \eqref{eq:minimise_hamiltonian_relax} is  analogous to
 the mirror descent algorithm  
for 
 optimisation problems
\cite{beck2003mirror},
with the Hamiltonian $H^\theta$
  playing the role of the gradient in mirror descent.

The policy ${\nu}^{\theta}_\varrho$ in \eqref{eq:minimise_hamiltonian_relax}
is Gaussian 
if the previous  policy 
$\nu^{m}$ is Gaussian. 
Indeed, 
let 
$\nu^{m}(t,x)=  \cN({k}^{m}_t +{K}^{m}_t x, (\lambda^{m}_t)^2)$
for some $(t,x)\in [0,T]\t \sR$,
then 
 one can directly verify    that 
\bb\label{eq:ppo_rho}
{\nu}^{\theta}_\varrho(t,x)=  \cN(k^{m+1}_t +K^{m+1}_t x, ({\lambda^{m+1}_t})^2),
\quad     (t,x)\in [0,T]\t \sR,
\ee
where
\begin{equation} 
\label{eq:MD_relaxed}
\begin{alignedat}{2}
{k}^{m+1}_t
&=h_t^m k_t^{ \theta}+(1-h_t^m) {k}^{m}_t,
\q 
&&
\quad 
{K}^{m+1}_t
=h_t^m K_t^{ \theta}+(1-h_t^m) {K}^{m}_t,
\\
({\lambda}^{m+1}_t)^{-2} & = \frac{2R_t}{\varrho} + (\lambda^{m}_t)^{-2},
\q 
&& \quad 
h^m_t = \frac{2R_t}{2R_t +  \varrho ({\lambda}^{m}_t)^{-2}},
\end{alignedat}
\end{equation} 
with $K_t^{ \theta}$ and $k_t^{ \theta}$
defined in \eqref{eq:K_k_theta}.

\end{Approach}
 
 Observe that 
in the present LQ setting,
solving the relaxed control problems in Approaches \ref{ex:exploration} and 
\ref{ex:MD} 
leads to    Gaussian-measure-valued policies.
The mean of these Gaussian measures 
is affine in the state  
and 
 depends     
on the estimated parameter $\theta$,
while the variance of these Gaussian measures 
depends on   the regularising weight $\varrho$. 
 In the following, 
 we demonstrate 
 how to
 follow   these Gaussian  policies
 to 
  interact with the dynamics 
\eqref{eq:LQ_RL_state},
  and how to choose the regularising weight
  for an efficient learning algorithm.

 \subsection{Execution of Gaussian relaxed policy}
 \label{sec:main_executation}
 This section studies the execution  of 
 a Gaussian relaxed policy via randomisation. 
 The precise definition 
 of Gaussian policies 
 is given as follows.

\begin{Definition}
\label{def:Gaussian_relax}
For each  
$k, K \in L^\infty([0,T];  \sR)$ and $\lambda \in L^\infty([0,T]; \sR_{\ge 0})$,
we define  a Gaussian  relaxed policy  $\nu:[0,T]\t \sR \to \cP(\sR)$, denoted by $\nu \sim \cG(k,K,\lambda)$,
such that 
for all $(t,x)\in [0,T]\t \sR$,
$\nu(t,x)= \cN(k_t+K_tx, \lambda_t^2)$.


\end{Definition}

{In the following, we introduce sampling procedures to execute a fixed Gaussian   policy $\nu\sim \cG(k, K,\lambda)$, where the coefficients $(k,K,\lambda)$ are known to the agent. This is the task encountered in each episode of the previously mentioned episodic learning problem, where a Gaussian policy $\nu$ has been constructed based on estimated parameters.
We allow the coefficients of the Gaussian policy to be arbitrary  bounded functions,
which include the polices \eqref{eq:optimal_relax_theta_rho} and  \eqref{eq:ppo_rho}
as special cases. 
Note that executing a Gaussian policy does not require knowing the true parameter $\theta^\star$
in \eqref{eq:LQ_RL_cost}, as the agent can simply take feedback controls based on observed system states.}  

To execute a Gaussian relaxed policy 
$\nu$, 
 at each given time and state,
the agent will 
sample a control action based on the distribution 
$\nu(t,x)\sim \cN(k_t+K_tx, \lambda_t^2)$. 
This   suggests 
interacting with the dynamics \eqref{eq:LQ_RL_state}
using 
 a randomised policy 
$\varphi(t,x)=k_t+K_tx+ \lambda_t \xi_t$,  
where 
  $\xi_t \sim \cN(0,1)$
is an injected Gaussian noise independent of the  Brownian motion $W$ in \eqref{eq:LQ_RL_state}. 
{The precise definition of the class of  randomised policies is given below.

\begin{Definition}
\label{definition:randomised}
For each 
 $k, K   \in L^\infty([0,T];  \sR)$, $\lambda \in L^\infty([0,T];  \sR_{\ge 0})$,
 and   measurable process $\xi:\Om\t [0,T]\to \sR$ 
 satisfying 
 for all $t \in [0,T]$,
 $\xi_t$ is an $\cF_0$-measurable standard normal random variable, 
 we define 
  a randomised policy
  with Gaussian noise
  $\varphi:\Om\t [0,T]\t \sR\to \sR$,
  denoted by $\varphi \sim \cR(k, K, \lambda,\xi)$, 
  such that  
  for all $(\om,t,x)\in \Om\t [0,T]\t \sR$,
 $\varphi(\om,t,x)=k_t+K_t x+\lambda_t \xi_t(\om)$.

\end{Definition}

 Note that 
 the process $\xi$ in Definition \ref{definition:randomised}  
models  additional  Gaussian noises 
used for executing a relaxed policy,
and 
is assumed, without loss of generality, 
sampled at $t=0$
(and hence independent of the Brownian motion $W$).
However, Definition \ref{definition:randomised} 
  only imposes a Gaussianity of 
$\xi_t$ for each $t\in [0,T]$,
without requiring 
 an independence  of $(\xi_t, \xi_s)$
for different $t,s\in [0,T]$.
In the sequel,
we simply refer to 
$\varphi\sim \cR(k, K, \lambda,\xi)$
as
a randomised policy if no confusion occurs.

To  compare the performances of the Gaussian relaxed policy  and a randomised policy,
we introduce the following cost of a  relaxed policy $\nu\sim \cG(k, K,\lambda)$
as in   \cite{wang2020reinforcement} (cf.~\eqref{eq:cost_relax_entropy}): 
\begin{align}
\label{eq:cost_RL_relax}
    \widetilde{J}^{\theta^\star}_0(\nu) \coloneqq\sE\Bigg[
    \int_0^T \int_\sR f(t,X^{\theta^{\star}, \nu}_t,a)\nu(t,X^{\theta^{\star}, \nu}_t; \d a)\,\d t
    +g(X^{\theta^{\star}, \nu}_T)
    \bigg],
\end{align}
where  $f$ and $g$ are  defined in \eqref{eq:learning_cost},
 and $X^{\theta^{\star}, \nu}$ is the solution to the following relaxed dynamics:
\begin{equation}
    \label{eq:LQ_RL_state_relax}
    \d X_t = 
   \int_\sR   (A^\star X_t +B^\star a)\nu(t,X_t; \d a)\, \d t
     +  
    \bar{\sigma}_t\,
     \d W_t,
     \quad  t\in [0,T]; \quad X_0 = x_0.
\end{equation}
Note that if $\lambda\equiv 0$, then $\nu (t,x)=\delta_{\phi(t,x)}$ and $ \widetilde{J}^{\theta^\star}_0(\nu) = J^{\theta^\star}(\phi)$, where  $\phi(t,x)=k_t+K_t x$, 
 $\delta_a$ is the Dirac measure supported at $a\in \mathbb{R}$, 
and $ J^{\theta^\star}(\phi)$ is  defined in \eqref{eq:LQ_RL_cost}. In other words, 
\eqref{eq:cost_RL_relax}-\eqref{eq:LQ_RL_state_relax}
incorporates 
\eqref{eq:LQ_RL_cost}-\eqref{eq:LQ_RL_state} as a special case. 
We also define the expected cost of a randomised policy 
  $\varphi\sim \cR(k,K,\lambda,\xi)$  (with fixed realisation of $\xi$):
\begin{equation}
\label{eq:cost_randomise}
     J^{\theta^\star}(\varphi)
     \coloneqq\sE\Bigg[
    \int_0^T  f\big(t,X^\varphi_t,\varphi(\cdot,t,X^\varphi_t)\big)\,\d t
    +g(X^\varphi_T)\,\bigg\vert\, \cF_0
    \bigg],
\end{equation}
where 
$X^\varphi$ satisfies the following controlled dynamics
(cf.~\eqref{eq:LQ_RL_state}): 
\begin{equation}
    \label{eq:sde_randomise}
    \d X_t = 
   (A^\star X_t+B^\star \varphi(\cdot,t,X_t))\, \d t
     +  
        \bar{\sigma}_t\,     \d W_t,
     \; t\in [0,T]; \quad X_0 = x_0.
\end{equation}
Note that 
the expectation in \eqref{eq:cost_randomise} is only taken over the Brownian motion $W$ in \eqref{eq:sde_randomise}.
Hence, 
$J^{\theta^\star}(\varphi)$ is a random variable, 
as its value  depends on the realisations of the injected noise $\xi$.
 
 The following lemma quantifies  the difference 
between $\widetilde{J}(\nu)$
with
$\nu\sim \cG(k, K ,\lambda)$
and 
${J}(\varphi)$
with
$\varphi\sim \cR(k, K, \lambda,\xi)$.

\begin{Lemma}
\label{lemma:cost_difference}
 Suppose (H.\ref{assum:lq_rl}) holds.
Then for all
$\nu\sim \cG(k, K ,\lambda)$ and
$\varphi\sim \cR(k, K, \lambda,\xi)$,
\begin{align}
\label{eq:cost_expansion}
\begin{split}
   J^{\theta^\star}(\varphi) - \widetilde{J}^{\theta^\star}_0(\nu) = \int_0^T h _{1,t} \lambda_t \xi_t\, \d t + \int_0^T h _{2,t} \lambda_t^2 (\xi_t^2 -1)\, \d t   + \int_0^T \int_0^T h _{3,t,r} \lambda_r \lambda_t \xi_r \xi_t \, \d r \d t,
    \end{split}
\end{align}
where $h _1,h _2 :[0,T]\to\sR$ and $h_3 :[0,T]^2\to\sR$
are some bounded functions whose sup-norms depend  only on the sup-norms of
 $k,K$ and
 the coefficients in (H.\ref{assum:lq_rl}).
\end{Lemma}

The proof of Lemma \ref{lemma:cost_difference} is given in Section \ref{sec:execution_general}.

Note that    the first two terms  on the right-hand side of  \eqref{eq:cost_expansion}  vanish under expectation,  while  the third term has a non-zero expectation if the injected noise  $\xi$ is correlated across time. 
Indeed, the following illustrative example 
shows  that repeated sampling a relaxed policy across multiple episodes may not yield the relaxed cost $\widetilde{J}^{\theta^\star}_0 (\nu)$, and therefore is not the correct execution of a relaxed policy
(cf.~the heuristic implementation
of a relaxed policy 
 suggested in \cite[Section 3]{firoozi2022exploratory}). 
The proof follows from a direct computation and the strong law of large numbers (omitted for brevity).
}

 \begin{Example}
 \label{example:repetition}
 Let $\theta^\star =(0,1)$, 
 $\mu, \sigma\in \sR$ and  $ \lambda\ge 0$.
Consider the relaxed policy $\nu:[0,1]\t \sR\to \cP(\sR)$
 such that 
 $\nu(t,x)=\cN(\mu x,\lambda^2)$
 for all $(t,x)\in [0,1]\t \sR$,
and  the   cost (cf.~\eqref{eq:cost_RL_relax}):
 $$
    \widetilde{J}^{\theta^\star}_0(\nu) =\sE\Bigg[
    \int_0^1 \int_\sR |a|^2\nu(t,X^\nu_t; \d a)\,\d t
    \bigg],
    \quad
\textnormal{with 
$ X^\nu_t = 
    \int_0^t \int_\sR a \nu(s,X^\nu_s; \d a)\, \d s+\sigma   W_t,
    \; \fa t\in [0,1].
$
}  
$$
 Let $(\zeta_i)_{i\in \sN}$ be  $\cF_0$-measurable independent standard normal random variables,
 and for each $i\in \sN$, let    $\varphi^i:\Om\t [0,1]\t \sR\to \sR$ 
be such that 
 $\varphi^i(\cdot,t,x)=\mu x+\lambda \zeta^i$
   for all $(t,x)\in  [0,1]\t \sR$, and
the cost  (cf.~\eqref{eq:cost_randomise}):
$$   J^{\theta^\star} (\varphi^i)
     \coloneqq\sE\Bigg[
    \int_0^1  \left(\varphi^i(\cdot,t,X^i_t)\right)^2\,\d t\,\bigg\vert\, \cF_0
    \bigg],
        \quad
\textnormal{with 
$ X^i_t = 
    \int_0^t \varphi^i(\cdot,s,X^i_s) \, \d s+\sigma   W_t,
    \; \fa t\in [0,1]
$.
}
$$
Then 
$
 \lim_{N\to \infty}
 \frac{1}{N}\sum_{i=1}^N
 J^{\theta^\star}(\varphi^i)
 =\widetilde{J}^{\theta^\star}_0(\nu)+
 \frac{\lambda^2}{2\mu}
 \left(
 e^{2 \mu}-2\mu-1
 \right)$ a.s.
%

 \end{Example}
 
Instead, we show that the relaxed cost \eqref{eq:cost_RL_relax} can be realised by sampling the target policy with sufficient independence across different time points within a single episode. Specifically, we consider executing a relaxed policy on a time grid using piecewise constant independent Gaussian noises, defined as follows.
 
\begin{Definition}
\label{definition:pc_randomised}

For each  
$\nu \sim \cG(k,K,\lambda)$
 and grid 
 $\pi=\{0=t_0<\cdots<t_N=T\}$ 
with   mesh size  $|\pi|=\max_{i=0,\ldots, N-1}(t_{i+1}-t_i)$,
we say   
$\varphi: \Om\t [0,T]\t \sR\to \sR$
is  a piecewise constant  independent
execution of $\nu$
on $\pi$, 
denoted by  $\varphi\sim \cR(k,K,\lambda,\xi^\pi)$,
 if there exist $\cF_0$-measurable 
independent standard normal random variables
$(\zeta_i)_{i=0}^{N-1}$ 
 such that 
  for all $(\om,t,x)\in \Om\t [0,T]\t \sR$,
 $\varphi(\om,t,x)=k_t+K_t x+\lambda_t \xi^\pi_t(\om)$,
with
 $\xi^\pi_t(\om)=\sum^{N-1}_{i=0} \zeta_i(\om)\bm{1}_{t_i\le t< t_{i+1}}$.
\end{Definition}

{  
\begin{Remark}
At first glance, it seems natural to consider injecting an exploration noise $\xi=(\xi_t)_{t\in [0,T]}$ 
that consists of 
(uncountably many)
pairwise independent standard Gaussian random variables. However, 
imposing pairwise independence over the    index set $[0,T]$ leads to highly irregular sample paths of $\xi$,  which   introduces  measure-theoretical issues. 
For instance,      \cite[Corollary 4.3]{sun2006exact} shows that, if $\xi$ is  
 a process of   (essentially) pairwise independent non-constant random variables, 
 then almost all sample paths of $t\mapsto \xi_t(\omega)$ are not   Lebesgue measurable\footnotemark.  
  This implies that   the time integrals of $\xi$ in Lemma \ref{lemma:cost_difference}   become  ill-defined,
  which prevents us from using 
  $\xi$    as  an exploration noise.

\footnotetext{
{A measurable process $\xi$ of   (essentially) pairwise independent random variables
 does not exist on a usual product probability space,
 but 
can be constructed on an extended product   space,
as shown in \cite{sun2006exact}. 
In particular, 
in \cite[Proposition 5.6]{sun2006exact},
given  the   index set $I=[0,1]$ and  a probability space $(\Lambda, \mathcal{F},P)$,
 the author constructs an extended product probability space $(I\times \Lambda, \mathcal{I} \boxtimes \mathcal{F}, \lambda\boxtimes P)$ (known as the rich Fubini extension), on which   a continuum of (essentially) i.i.d.~random variables is defined. Here $\mathcal{I}$ and $\lambda$ are   $\sigma$-algebra and probability measure on $I$, respectively, 
whose  specific choices  are part of the construction.
 The implication of \cite[Corollary 4.3]{sun2006exact} is that  $\lambda$ 
 cannot be chosen as   the Lebesgue measure
 in this context.
 }
}

 Due to the technical challenges associated with the use of a ``continuum" of i.i.d.~random variables,  we   have opted for a   relaxed approach 
 by ensuring 
  sufficient independence on a grid,
and adjusting  the sampling frequency through the mesh size. 
This allows us to bypass    these measure-theoretical difficulties while still achieving the desired level of independence in our exploration noises.

It is worth noting that one can also consider changing the coefficients of   randomised policies    only at   grid points. This will  introduce  an additional time discretisation error, which can be controlled based on the time regularity of   $(k, K, \lambda)$.
Discrete-time policies of this nature have been analysed in \cite[Section 2.3]{basei2022logarithmic} for LQ-RL problems with deterministic policies. The analysis therein shows that   discrete-time policies achieve similar regret as continuous-time policies, with an additional term that explicitly depends on the time step sizes used in the algorithm.
We anticipate that a similar regret analysis can be performed in the present setting with stochastic policies,
whose detailed analysis is left for future work.

\end{Remark} 
}

The following theorem 
establishes the convergence  of 
$J^{\theta^\star}(\varphi) - \widetilde{J}^{\theta^\star}_0(\nu)$
with respect to  the mesh size $|\pi|$ and the standard deviation  
$\lambda $,
{where  $\nu\sim \cG(k, K ,\lambda)$
and 
 $\varphi\sim \cR(k,K, \lambda, \xi^\pi)$ are given in Definitions
\ref{def:Gaussian_relax} and 
  \ref{definition:pc_randomised}, respectively.}
This consequently 
proves that $\varphi$ is   a proper execution of $\nu$. 
The explicit dependence of the convergence rate on $\lambda$ is crucial for determining the appropriate choice of entropy regularisation weight and for analysing the algorithm performance   in Section \ref{sec:regret_analysis}. 

%

\begin{Theorem}
\label{Thm:execution gap_drift_control}
 Suppose (H.\ref{assum:lq_rl}) holds and let 
$\nu\sim \cG(k, K ,\lambda)$.
Then there exists a constant $C\ge 0$, depending only on the sup-norms of
 $k,K$ and
 the coefficients in (H.\ref{assum:lq_rl}),
such that 
\begin{enumerate}[(1)]
\item 
\label{item:gap_high_prob}
for all grids
$\pi$ of $[0,T]$,
all 
  $\varphi\sim \cR(k,K, \lambda, \xi^\pi)$ and all $\delta\ge 0$,
$$
\sP\left(
\left|J^{\theta^\star}(\varphi)-\widetilde{J}^{\theta^\star}_0(\nu)
\right|
\geq C \sqrt{|\pi|} 
 \|\lambda\|_\infty(1+
\|\lambda\|_\infty)
\Big( 1 +  \ln\big(\tfrac{1}{\delta}\big)\Big)
\right) \leq \delta,
$$
\item 
\label{item:gap_as}
if 
 $\pi_n=\{i\frac{T}{n}\}_{i=0}^n$
 and 
$\varphi^n\sim \cR(k,K, \lambda, \xi^{\pi_n})$
for all  $n\in \sN$,
then 
it holds a.s.~that
$$
\big|J^{\theta^\star}(\varphi^n)-\widetilde{J}^{\theta^\star}_0(\nu)
\big|
\le C  n^{-\frac{1}{2}}({\ln n })
 \|\lambda\|_\infty(1+
\|\lambda\|_\infty),
\quad   \textnormal{for all sufficiently large $n\in \sN$,}
$$
\item
\label{item:gap_exp}
for all grids
$\pi$ of $[0,T]$
and  
  $\varphi\sim \cR(k,K, \lambda, \xi^\pi)$,
  $|\sE[J^{\theta^\star}(\varphi)-\widetilde{J}^{\theta^\star}_0(\nu)]|\le C|\pi|\|\lambda\|^2_\infty$.
\end{enumerate}

\end{Theorem}
The proof of Theorem \ref{Thm:execution gap_drift_control} is given in Section \ref{sec:execution_general}. 

\begin{Remark}
Roughly speaking,
Theorem \ref{Thm:execution gap_drift_control}
shows that if a randomised policy  $\varphi\sim \cR(k,K, \lambda, \xi)$
is driven by an external noise $\xi$
with  independent marginals across time,
then its cost $J^{\theta^\star}(\varphi)$ well represents the 
relaxed cost $\widetilde{J}^{\theta^\star}_0(\nu)$. 
Such a process $\xi$ is closely related to the white noise in the distribution theory,
and we refer the reader to Section \ref{sec:discussion} for a detailed discussion. 

\end{Remark}

 \subsection{
 Tuning  entropy regularising weight
 for optimal regret}
 \label{sec:regret_analysis}
 
Based on Theorem \ref{Thm:execution gap_drift_control}, this section    analyses 
the episodic learning algorithms 
   outlined in Sections \ref{sec:lq_rl_outline} and \ref{sec:relaxed_control_problem}. 
In particular, 
it
optimises the policy sampling frequency and entropy regularisation weight
to achieve the best possible algorithm regret.

\paragraph{Probabilistic learning framework.}
We start by recalling the generic  
 learning procedure for solving \eqref{eq:LQ_RL_cost}-\eqref{eq:LQ_RL_state} with unknown $\theta^\star$: 
  for each episode $m\in \sN$, let  $\nu^{m}=  \cN({k}^{m}_t +{K}^{m}_t x, (\lambda^{m}_t)^2)$
 be a given Gaussian relaxed policy.
Here
$ {k}^{m}, {K}^{m} , \lambda^{m} :\Om\t [0,T]\to \sR$
are   
designed  based on   observations
from previous episodes.
The  agent  then 
 chooses 
 a grid $\pi_m$ of $[0,T]$
 and   executes 
$  {\nu}^{m}$
 on   $\pi_m$
via  a  
 policy $ {\varphi}^{m}  $:
for all  $(\om, t,x)\in \Om\t    [0,T]\t \sR$,
\begin{equation}
\label{eq:randomised_control_m_penalty}
{\varphi}^{m} (\om, t,x) =
 {k}^{m}_t(\om)+   {K}^{{m}}_t(\om) x + {\lambda}^{m}_t (\om) 
\xi^{\pi_m}_t(\om),
 \end{equation}
where  $\xi^{\pi_m}$ is 
a piecewise constant process  on $\pi_m$
generated  by independent standard normal variables $(\zeta^m_i)_{i \in \sN}$
 (cf.~Definition \ref{definition:pc_randomised}).
The agent     observes 
a trajectory of 
the corresponding state process $X^{m}$
governed by the  following dynamics (cf.~\eqref{eq:LQ_RL_state}):
\begin{equation}
     \label{eq:LQ_RL_state_m}
    \d X^m_t = 
    \theta^\star  Z^{   m}_t\, \d t
     +  
    \bar{\sigma}_t\,
     \d W^m_t,
     \quad X^m_0 = x_0,
     \q 
     \textnormal{with 
    $ Z^{   m}_t
=\begin{psmallmatrix}
X^{m}_t
\\
{\varphi}^{m} (\cdot, t, X^{m}_t)
\end{psmallmatrix}$,
     }
\end{equation}
where $W^m$ 
is   an independent   Brownian motion corresponding to the $m$-th episode.\footnotemark
 \footnotetext{Without loss of generality,
 we assume that
 $(\xi^{\pi_m}, W^{m})_{m\in \sN}$
 are defined on the probability  space 
$(\Om,\cF,\sP)$,
$(\xi^{\pi_m})_{m\in \sN}$ are measurable with respect to 
  the initial $\sigma$-algebra
   $\cF_0$,
   and  all 
  Brownian motions $(W^m)_{m\in\sN}$  are 
   independent.
}
The agent 
then  
constructs an estimate of 
 $\theta^\star$,
denoted by $\theta_m$, 
   based on  all observations 
$(X^{ n},\xi^{\pi_n})_{n=1}^m$,
 and 
 solves a regularised  relaxed control problem 
 (see  e.g., Approaches \ref{ex:exploration} and 
\ref{ex:MD})
 based on $\theta_m$ and  a properly chosen   weight $\varrho_m>0$. 
This 
determines 
the relaxed policy $\nu^{m+1}$ 
for the next  episode.

\paragraph{Bayesian inference.}
To derive a concrete learning algorithm, we infer the  parameter $\theta^\star$
by a  Bayesian approach as in \cite{szpruch2021exploration}. 
Let
 $\cF^{\textrm{ob}}_m$ be
  the   observation $\sigma$-algebra
  generated by 
the   state processes and 
the  injected noises 
after the $m$-th episode:
\bb
\label{eq:observation}
\cF^{\textrm{ob}}_m=\sigma\{X^{n}_t, \xi^{\pi_n}_t\mid t\in [0,T],n=1,\ldots,  m\}\vee \cN_0, \quad m\in \sN\cup\{0\},
\ee
where $\cN_0$ is the $\sigma$-algebra generated by $\sP$-null sets.
By \cite[Section 7.6.4]{liptser1977statistics},
for each $m\in \sN$,
the likelihood function of $   {\theta}^\star$
with observations $(X^{  m}, \xi^{\pi_m})$
is given by   
\begin{align}
\label{eq:likelihood}
&
\ell(  {\theta}^\star  \mid  X^{  m}, \xi^{\pi_m})
 \propto
 \exp\bigg(\int_0^T
 \frac{1}{\bar{\sigma}^{2}_s}
   {\theta}^\star Z^{ m}_s  \d X^{  m}_s 
- \frac{1}{2} \int_0^T
 \frac{1}{\bar{\sigma}^{2}_s}(   {\theta}^\star Z^{  m}_s)^2 \d s
\bigg),
\end{align}
where 
$\propto$ stands for   proportionality  up to a constant independent of $  {\theta}^\star$.
 Hence, 
given 
the initial belief  that $   {\theta}^\star$ follows the prior distribution 
$\pi(  {\theta}^\star\mid  \cF^{\textrm{ob}}_0 )=\cN({\theta}_0,  V_0)$ for some $\theta\in \sR^{1,2}$ and $V_0\in \sS_+^2$,
 the posterior distribution of $\pi(   {\theta}^\star \mid  \cF^{\textrm{ob}}_m)$ after the $m$-th episode is given by
\begin{align}
\label{eq:posterior}
\begin{split}
 \pi(  {\theta}^\star  \mid \cF^{\textrm{ob}}_m)
& \propto \pi(   {\theta}^\star \mid \cF^{\textrm{ob}}_0) \prod_{n = 1}^m \ell(   {\theta}^\star \mid  X^{ n}, \xi^{\pi_n})
 \propto  \exp\bigg( -\frac{(   {\theta}^\star-\hat{{\theta}}_m)  V_m^{-1}
(   {\theta}^\star-\hat{{\theta}}_m)^\top }{2}  
\bigg),
\end{split}
\end{align}
where 
$\hat{\theta}_m$ and 
$V_m$ are 
given by:
 \begin{subequations}
\label{eq: statistics}
\begin{alignat}{1}
V_m
&\coloneqq
\bigg(V_0^{-1}+
\sum_{n=1}^m
\int_0^T
\frac{1}{\bar{\sigma}^{2}_s} Z^{  n}_s (  Z^{ n}_s)^\top \,\d s
\bigg)^{-1}\in \sS^{2}_+,
\label{eq:posterior_variance}
\\
\hat {\theta }_{m}
&\coloneqq
\bigg(
{\theta}_0
V_0^{-1}+
\sum_{n=1}^m
\left(
 \int_0^T
 \frac{1}{\bar{\sigma}^2_s}
 Z^{  n}_s
 \,\d X^{n}_s
 \right)^\top
 \bigg)V_{ m}
  \in \sR^{1\t 2}.
  \label{eq:posterior_mean}
\end{alignat}
\end{subequations}

For simplicity,
we   assume that the agent 
 has  prior knowledge of the magnitude of  $\theta^\star$:
 \begin{Assumption}
\label{eq:bounded}
$\Theta$ is a compact    subset of $ \sR^{1\t 2}$
such that $\theta^\star\in \operatorname{int}(\Theta)$.
\end{Assumption}

Condition (H.\ref{eq:bounded})  suggests  
a  
 truncated maximum a posteriori 
estimate $\theta_m$ as follows:
 \bb
   \label{eq:theta_m_project}
 \theta_m  \coloneqq \Pi_\Theta (\hat{\theta}_m,V_m)\in \sR^{1\t 2},
 \ee
 where 
  $\Pi_\Theta: \sR^{1\t 2}\t \sS^2_+\to \Theta$
is a measurable function
such that  $\Pi_\Theta(\theta,V)=\theta$ for all $\theta\in \Theta$ and $V\in \sS^2_+$.
 Possible choices of $\Pi_\Theta$
include the orthogonal projection for  closed and convex $\Theta$,
or a    function that maximises the posterior distribution 
\eqref{eq:posterior} over all ${\theta}^\star\in \Theta$.
We emphasise that  the 
a-priori bound of $\theta^\star$ 
is only imposed to simplify the analysis,
which 
 along with   the projection step  $\Pi_\Theta$ 
ensures  the estimates  $(\theta_m)_{m\in \sN}$ are a-priori bounded.
In the case without prior knowledge on $\theta^\star$, 
one can remove  the projection step, and instead adopt  a sufficiently long initial exploration step to guarantee a uniform   bound  of   $(\hat{\theta}_m)_{m\in \sN}$
with high probability 
(see e.g., \cite{simchowitz2020naive,basei2022logarithmic}).

\paragraph{Performance measure.}
  
 In the above learning procedure,
 the agent solves suitable 
  regularised relaxed control problems  
  and executes the relaxed policies through randomisation.
  To measure the performance 
 of a learning algorithm in this setting,
we consider  the  regret of learning  defined as follows (see e.g., \cite{guo2021reinforcement, szpruch2021exploration,basei2022logarithmic}):
 for each $N\in\sN$,
 \begin{equation}
 \label{eq:regret_exploration}
     \textrm{Reg}(N)
     =\sum_{m=1}^N
     \left( J^{{\theta}^\star}(
     \varphi^{{m }})
     -J^{{\theta}^\star}(\phi^{\theta^\star})\right),
 \end{equation}
 where $J^{{\theta}^\star}(\phi^{\theta^\star})$ 
is the optimal cost that agent can achieve knowing the parameter $  \theta^\star$ (see \eqref{eq:LQ_RL_cost}),
and 
 $J^{{\theta^\star}}(
     \varphi^m)$ is the expected cost of the  randomised policy
     $\varphi^m$ for the $m$-th episode:
\begin{equation}
\label{eq:cost_m_episode}
    J^{{\theta^\star}}(
     \varphi^m)
     \coloneqq
     \sE\bigg[
     \int_0^T  f
     \left(t,X^m_t,
     \varphi^m(\cdot, t,X^m_t)\right)\,\d t
    +g(X^m_T)
     \,\bigg\vert\,  \cF^{\textrm{ob}}_{m-1}\vee \cZ_m\bigg],
\end{equation} 
where
$\cF^{\textrm{ob}}_{m-1}$
is defined as in \eqref{eq:observation},
and $\cZ_m=  \sigma\{\xi^{\pi_m}_t\mid t\in [0,T]\}$.
The regret
$\textrm{Reg}(N )$
characterises the cumulative loss from taking sub-optimal policies up to the $N$-th episode. 
Agent's aim is to construct   policies 
 whose regret grows sublinearly with respect to $N$.

Note that 
 the expectation
 in \eqref{eq:cost_m_episode}
is only taken  with respect to the Brownian motion $W^m$, and consequently,
for each $m\in \sN$,
$J^{{\theta^\star}}(
     \varphi^{m})$
is a random variable depending on 
the realisations of
the Brownian motions $(W^n)_{n=1}^{m-1}$ and 
the injected noises $(\xi^{\pi_n})_{n=1}^m$. 
   

\paragraph{Regret analysis.}
Given the parameter inference scheme \eqref{eq: statistics}-\eqref{eq:theta_m_project},
 the performance   of  a learning algorithm  
 now 
depends on  
the  construction 
and execution 
of    relaxed policies.
In the sequel,
we propose two algorithms based on  regularised   control problems 
in Approaches \ref{ex:exploration} and 
\ref{ex:MD},
and  choose   the regularising  weights   $(\varrho_m)_{m\in \sN}$
and   execution grids $(\pi_m)_{m\in\sN}$ 
to minimise the  algorithm regrets.

 The first algorithm computes     relaxed policies 
 by solving the  relaxed control problem \eqref{eq:cost_relax_entropy}
with   exploration rewards. 
 We summarise the algorithm as follows.

\begin{algorithm}[H]
\label{alg:exploration}
\DontPrintSemicolon
\SetAlgoLined

  \KwInput{
  $\theta_0\in \sR^{1\t 2}$,
  $V_0\in \sS^{2}_+$, 
   a truncation function $\Pi_\Theta$,
   and a  Gaussian relaxed policy $\nu^1$.
  }

 
 \For{$m = 1, 2,\ldots$}
{
	{
	Determine   a grid $\pi_m$ of $[0,T]$ and 
	exercise  $\nu^{m}$ on $\pi_m$ via 
	$\varphi^m$ as  in \eqref{eq:randomised_control_m_penalty}. 
	}\;
	 {Obtain the updated estimate  $\theta_m$ via \eqref{eq:theta_m_project}.}\;
	{
	Determine  $\varrho_m>0$  and 
	compute  $\nu^{m+1}$ 
	by \eqref{eq:optimal_relax_theta_rho}
	with $\theta=\theta_m$ and $\varrho=\varrho_m$.
	}
	\label{step:update_nu}
	\;

 }
 \caption{Learning with  exploration reward}
\end{algorithm}

The following theorem chooses 
$(\varrho_m)_{m\in\sN}$
 and $(\pi_m)_{m\in \sN}$
 such that Algorithm \ref{alg:exploration}
 achieves  a regret of the order
 $\cO(\sqrt{N} \ln N)$ in expectation.
This recovers the   regret  order 
of the phased-based algorithm 
with deterministic policies 
in \cite{szpruch2021exploration}. 
 Note that
  the additional exploration reward vanishes 
 as the number of episodes tends to infinity.

\begin{Theorem}
\label{thm:vanish_exploration}
Suppose (H.\ref{assum:lq_rl}) and (H.\ref{eq:bounded}) hold. 
Let $\theta_0\in \sR^{1\t 2}$, 
$V_0\in \sS_+^2$, 
$\varrho_0>0$ 
and  $\nu^1\sim \cG(k^{\theta_0},K^{\theta_0}, \sqrt{{\varrho_0}/{(2R)}})$. 
 Then there exists   $c_0 >0$
 such that if one sets 
$ \varrho_m= \varrho_0 m^{-\frac{1}{2}} \ln (m+1)
$ 
and $
|\pi_m|  \leq c_0$ 
for all  $m\in \sN$,
then there exists a constant $C \geq 0$ such that 
the regret of Algorithm \ref{alg:exploration} satisfies
$$   \sE \big[ \textnormal{Reg}(N) \big]
     \le  C\sqrt{N} \ln N,
     \quad \fa N\in \sN\cap[2,\infty). 
$$ 
\end{Theorem}

\begin{Remark}
The  proof of  Theorem \ref{thm:vanish_exploration}
  is given in Section  \ref{sec:analysis_exploration},
where 
the key steps are outlined
at the end of this section  
for the reader's convenience.
Note that for   simplicity  
we choose   $\nu^1$ 
 based on  $\theta_0$, 
 but the same regret order can be achieved
with
    $\nu^1\sim \cG(k^1,K^1, \lambda^1)$
 for any deterministic 
$ k^1, K^1$ and $\lambda^1$.
Moreover, 
we focus on  optimising the regret order in expectation,
and obtain  
a   simple deterministic  scheduling 
of   the sampling frequency    $(|\pi_m|)_{m\in \sN}$ 
and   regularising weights $(\varrho_m)_{m\in \sN}$.
The analysis can be extended to ensure that
a  similar  regret bound
holds  with probability at least $1-\delta$,
 for any given  $\delta\in (0,1)$,
 but 
$(\pi_m)_{m\in \sN}$ and $(\varrho_m)_{m\in \sN}$
have to be chosen
depending explicitly on $\delta$
(see e.g., \cite{
simchowitz2020naive,
guo2021reinforcement,
basei2022logarithmic}).
We leave a rigorous analysis of such a high-probability regret 
for future research.

 \end{Remark}
 
 The second algorithm 
 computes     relaxed policies 
 with 
 the proximal
 policy update
\eqref{eq:minimise_hamiltonian_relax}.
 We summarise the algorithm as follows.
   
  \begin{algorithm}[H]
\label{alg:MD}
\DontPrintSemicolon
\SetAlgoLined

  \KwInput{
 $\theta_0\in \sR^{1\t 2}$,
  $V_0\in \sS^{2}_+$, 
   a truncation function $\Pi_\Theta$,
   and a  Gaussian relaxed policy $\nu^1$.
  }

 
 \For{$m = 1, 2,\ldots$}
{
	{
	Determine   a grid $\pi_m$ of $[0,T]$ and 
	exercise  $\nu^{m}$ on $\pi_m$ via 
	$\varphi^m$ as  in \eqref{eq:randomised_control_m_penalty}. 
	}\;
	 {Obtain the updated estimate  $\theta_m$ via \eqref{eq:theta_m_project}.}\;
	{
	Determine  $\varrho_m>0$  and 
	compute  $\nu^{m+1}$ 
	by \eqref{eq:MD_relaxed}
		with $\theta=\theta_m$ and $\varrho=\varrho_m$.
	}\;
  
 }
 \caption{Learning with 
 proximal
 policy update}
\end{algorithm}

The following theorem  shows that    Algorithm \ref{alg:MD} achieves a similar  sublinear regret  as 
Algorithm \ref{alg:exploration}.
Note that unlike 
Algorithm \ref{alg:exploration},
 the regularising weight $\varrho_m$ 
 tends to infinity as the learning proceeds. 
 
 \begin{Theorem}
\label{thm:MD}
Suppose (H.\ref{assum:lq_rl}) and (H.\ref{eq:bounded}) hold. 
Let $\theta_0\in \sR^{1\t 2}$, 
$V_0\in \sS_+^2$, 
$\varrho_0>0$ 
and  $\nu^1\sim \cG(k^{\theta_0},K^{\theta_0}, \varrho_0)$. 
 Then there exists   $c_0 >0$
 such that if one sets 
 $ \varrho_m= \varrho_0 m^{\frac{1}{2}} \ln (m+1) $
 and 
$|\pi_m|  \leq c_0$
for all  $m\in \sN$,
then there exists a constant $C \geq 0$ such that 
the regret of Algorithm \ref{alg:MD} satisfies
$$   \sE \big[ \textnormal{Reg}(N) \big]
     \le  C\big(  \sqrt{N} (\ln N)   \ln(\ln N)  \big),
     \quad \fa N\in \sN\cap[3,\infty). 
$$ 

 \end{Theorem}

\begin{Remark}
The proof of Theorem \ref{thm:MD} is given in Section \ref{sec:analysis_MD},
whose key steps are outlined below. 
The argument   is more  involved  than 
that of     Theorem \ref{thm:vanish_exploration}.
This is due to the fact that  
 $  {\nu}^{{m+1}}$, $m\in \sN$,
is a convex combination of all previous 
policies $ ( {\nu}^{{n}})_{n=0}^{m}$
(cf.~\eqref{eq:MD_relaxed}),
and hence 
the cost   $J^{\theta^\star}(\varphi^{m+1})$  
depends  explicitly on  $(\theta_{n})_{n=0}^{m}$.
This   memory dependence requires  
 quantifying the precise dependence 
of the performance gap $J^{\theta^\star}(\varphi^{m+1})-J^{{\theta}^\star}(\phi^{\theta^\star})$  
on the accuracy of $(\theta_{n})_{n=0}^{m}$.
\end{Remark}

\begin{proof}[Sketched proofs of Theorems \ref{thm:vanish_exploration} and \ref{thm:MD}]
The key idea  of  the regret analysis is to balance the exploration--exploitation trade-off
for Algorithms 
\ref{alg:exploration} and \ref{alg:MD}.
To this end, we  decompose
 the regret
into
 \begin{equation}
 \label{eq:Regret Decomposition}
     \textrm{Reg}(N)
     =\sum_{m=1}^N
     \left( J^{{\theta}^\star}(
     \varphi^{{m }}) 
     -J^{{\theta}^\star}(\phi^{m})\right) + \sum_{m=1}^N
     \left( J^{{\theta}^\star}(
     \phi^{{m }}) 
     -J^{{\theta}^\star}(\phi^{\theta^\star})\right),
 \end{equation}
 where $\phi^m(t,x)=k_t^m+K^m_t x$ with $(k^m,K^m)$ is defined in \eqref{eq:randomised_control_m_penalty}.
 The first term induces the exploration cost due to randomisation,
 and the second term
 describes
 the exploitation cost 
 caused by  the sub-optimal policies $(\phi^n)_{n=1}^N$.
 In particular, 
 we prove that 
 $\sE[J^{{\theta}^\star}(
     \varphi^{{m }})
     -J^{{\theta}^\star}(\phi^{m})\mid \cF_{m-1}^{\rm ob}]
=\cO( \| \lambda^m \|^2_\infty + \sqrt{|\pi_m|} 
 \|\lambda^m\|^2_\infty)   $
 and  
 $
  J^{{\theta}^\star}(
     \phi^{{m }}) 
     -J^{{\theta}^\star}(\phi^{\theta^\star})
=  \cO\Big( \|K^m - K^{\theta^\star} \|^2_\infty + \|k^m - k^{\theta^\star} \|^2_\infty\Big)$.
We then
 reduce
the   exploitation loss
   to  the parameter estimate error
by using Step \ref{step:update_nu} of Algorithms 
 \ref{alg:exploration} and \ref{alg:MD}, 
   and
 prove that 
 $ J^{{\theta}^\star}(
     \phi^{{m }}) 
     -J^{{\theta}^\star}(\phi^{\theta^\star})
   =\cO (|\theta_m - \theta^\star|^2)$ for  Algorithm \ref{alg:exploration}  and $J^{{\theta}^\star}(
     \phi^{{m }}) 
     -J^{{\theta}^\star}(\phi^{\theta^\star})
     = {\cO}\left( \sum_{n=1}^m \frac{1}{n}\left\|\frac{ \lambda^m}{\lambda^n} \right\|^2_\infty|\theta_n - \theta^\star|^2\right)$ for Algorithm \ref{alg:MD}
     (see Proposition \ref{prop: action gap}).
     We further quantify the estimation accuracy by 
     $ |\theta_m - \theta^\star |^2 = \cO \left( \frac{1}{\sum_{n=1}^m \min_{t \in [0,T]} (\lambda^m_t)^2 } \right)$,
     provided that the mesh size $|\pi_m|$ is sufficiently small
     (see Section \ref{sec:general_policy}).
 These analyses show that 
   the exploration cost  in \eqref{eq:Regret Decomposition}
 increases in  $(\lambda^n)_{n\in \sN}$
 and the exploitation loss 
 in \eqref{eq:Regret Decomposition}
 decreases 
 in $(\lambda^n)_{n\in \sN}$. 
 
 Observe that  
  for both Algorithms 
\ref{alg:exploration} and \ref{alg:MD}, the magnitude of 
 $\lambda^{m+1} $ is determined by the  weights $(\varrho_n)_{n=1}^m$
 (cf.~\eqref{eq:optimal_relax_theta_rho} and \eqref{eq:MD_relaxed}).
 Hence, by choosing the weights $(\varrho_m)_{m\in \sN}$
 as in Theorems \ref{thm:vanish_exploration}
and \ref{thm:MD},
we can balance the exploration 
and exploitation 
and obtain algorithms with optimal regret orders (up to a  logarithmic factor). 
\end{proof}

\subsection{Discussions}
\label{sec:discussion}

\paragraph{Random execution and scaled white noise.}
The  execution noise 
in Definition \ref{definition:pc_randomised}
is a properly scaled discrete-time  white noise. 
Discretised  white   noise  has  been used 
in \cite{delarue2021exploration}
for exploration. 
More precisely,
let $\phi$ be a deterministic policy that is 
affine in the state, 
let $\sW$ be an independent Brownian motion,
and 
 for each $n\in \sN$, let $\pi_n$ be a uniform grid with mesh size $\delta_n=T/n$.
Given  $\eps>0$, 
\cite{delarue2021exploration}   considers  the randomised policy  $\varphi_{n,\eps}=\phi+\eps\sum_{i=0}^{n-1}\tfrac{\sW_{(i+1)\delta_n}
-\sW_{i\delta_n}
}{\delta_n}\boldsymbol{1}_{[i\delta_n,(i+1)\delta_n)}$
for all large $n\in\sN$, 
and  the normalised cost
$J^{\theta^\star}(\varphi_{n,\eps})-Cn\eps^2$
for some constant $C>0$.
As 
$\varphi_{n,\eps}$ behaves similar to $\phi+\eps\dot{\sW}$
with $ \dot{\sW}$ being a white noise,  
the  term 
$n\eps^2$ is necessary for the finiteness of  
the normalised cost. Indeed, 
for a suitable choice of $C$,
 the normalised cost    converges to   $J^{\theta^\star}(\phi)$ 
   as $n\to \infty$.

Instead of fixing $\eps$ as in \cite{delarue2021exploration},
the piecewise constant   randomisation  in Definition \ref{definition:pc_randomised} adjusts $\eps$
according to the mesh size.
To see it, 
let $\lambda>0$ and 
for each $n\in \sN$, let $\eps_n=\lambda \sqrt{\delta_n}$. By Theorem \ref{Thm:execution gap_drift_control},
$J^{\theta^\star}(\varphi_{n,\eps_n})$ approximates 
the relaxed cost $\widetilde{J}^{\theta^\star}_0(\nu)$ with $\nu\sim \cN(\phi,\lambda)$. 
The normalisation term $n\eps_n^2=\cO(\lambda^2)$ corresponds to the exploration cost 
 $\widetilde{J}^{\theta^\star}_0(\nu)-J^{\theta^\star}(\phi)$
 in \eqref{eq:exploration_cost}
(see also \cite{wang2020reinforcement}).
 
\paragraph{Random execution and chattering lemma.}
In the control theory, the chattering lemma (see e.g., \cite{nicole1987compactification})
ensures that 
there exists a sequence of strict control processes converging to a relaxed control process. 
However,  the proof is   nonconstructive 
and  typically requires a compact action space. 
Also there is no     explicit  convergence rate. 

Our work focuses on the LQ setting and provides an explicit   approximation of 
  Gaussian relaxed policies.
We also quantify the precise approximation error in terms of the mesh size and the variance of the Gaussian relaxed policies. 

\paragraph{LQ-RL with controlled diffusion.}
For general LQ-RL problems whose state dynamics involves  controlled diffusions, 
the state process is no longer Gaussian.
Moreover, as 
 the noise of the state dynamics can degenerate, 
 the Bayesian estimation \eqref{eq: statistics}
may not be well-defined. In this case, one has to  
derive   learning algorithms in a problem dependent way. 

For instance, 
let $\phi$ be a given policy and 
consider  the  state process 
$X^\phi$ satisfying the following one-dimensional dynamics:
\begin{equation}
    \label{eq:LQ_RL_state_control vol}
    \d X_t = 
    (A X_t +B \phi(t,X_t))\, \d t
     +  
    (C X_t + D \phi(t,X_t) + \bar{\sigma}) \,
     \d W_t,
     \quad  t\in [0,T]; \quad X_0 = x_0.
\end{equation}
Note that 
by observing  $X^\phi$ continuously, the agent can infer the constants $(C,D,\bar{\sigma})$ using finite episodes. 
Indeed,  the agent can compute  the process 
$Y^\phi_t= C X^\phi_t + D \phi(t,X^\phi_t) + \bar{\sigma}$, $t\in [0,T]$, from 
  the rate of increase of the quadratic variation of $X^\phi$.
Then 
the agent can recover $D$ from the random variables
$Y^{\phi^{(i)}}_0$, $i=1,2,3$,
with  constant policies $\phi^{(1)}\equiv -1, \phi^{(2)}\equiv 0$, and $\phi^{(3)}\equiv 1$.
The constants $C$ and $\bar{\sigma}$ can further be inferred from   observed paths of $Y^\phi$,
by  executing a policy $\phi$ such that  paths of $X^\phi$ are non-constant almost surely. 

Given $(C,D,\bar{\sigma})$, the agent may eliminate the observation noise 
and  simplify the learning process. To see it, suppose that   $D,\bar{\sigma}\not= 0$. Then the agent will estimate $(A,B)$ by executing the policy $\phi(t,x) = -D^{-1}(Cx + \bar{\sigma})$. This leads to a deterministic state dynamics: 
$\d X^\phi_t = \theta Z^\phi_t \d t$, with $\theta=(A,B)$ and $Z^\phi_t \coloneqq  \begin{psmallmatrix}
X^\phi_t  \\ 
\phi(t,X^\phi_t)
\end{psmallmatrix}$.
As $\bar{\sigma}$ is non-zero and  $X^\phi_\cdot$ is  in general non-constant, 
$X^\phi_\cdot$ and $\phi(\cdot,X^\phi_\cdot)$ are linearly independent, 
and $\int_0^T Z^\phi_t (Z^\phi_t)^\top \, \d t$ is invertible. This allows to identify $\theta$ using one episode. 
A complete study of LQ-RL with controlled diffusion 
is 
left to future research.

 \section{Proofs of Lemma \ref{lemma:cost_difference}  and Theorem \ref{Thm:execution gap_drift_control}}
 \label{sec:execution_general}

{\begin{proof}[Proof of Lemma \ref{lemma:cost_difference}]
   Let $X^\nu$ be the solution to \eqref{eq:LQ_RL_state_relax} associated with  the  policy $\nu\sim \cG(k, K ,\lambda)$,
   and let $X^\varphi$
   be the solution to \eqref{eq:sde_randomise} associated with  the  policy   $\varphi\sim \cR(k, K, \lambda,\xi)$. Recall  that 
for all $t\in [0,T]$, 
   \begin{align*}
    \d X^\nu_t &= 
   (A^\star X^\nu_t+B^\star(k_t + K_t X^\nu_t) )\, \d t
     +  
        \bar{\sigma}_t\,     \d W_t,
 \quad X^\nu_0 = x_0, \\
      \d X^\varphi_t &= 
   (A^\star X^\varphi_t+B^\star(k_t + K_t X^\varphi_t + \lambda_t \xi_t) )\, \d t
     +  
        \bar{\sigma}_t\,     \d W_t,
        \quad X^\varphi_0 = x_0.
\end{align*}
Then  by Duhamel's principle,  for all $t\in [0,T]$, 
$$X^\nu_t = \Phi_t x_0 + \Phi_t \int_0^t \Phi_s^{-1} \big(B^\star k_s + \bar{\sigma}_s W_s \big) \d s, \quad   X^\varphi_t = X^\nu_t + \Phi_t \int_0^t \Phi_s^{-1} B^\star \lambda_s \xi_s \d s, $$
where $(\Phi_t)_{t \in [0,T]}$ is the solution to $\d \Phi_t = (A^\star + B^\star K_t) \Phi_t \d t$ with $\Phi_0 =1$. 
Therefore, for all  $t \in [0,T]$, 
\begin{align*}
\sE[ X^\varphi_t - X^\nu_t |\cF_0] &= \Phi_t \int_0^t \Phi_s^{-1} B^\star \lambda_s \xi_s \d s = \Phi_t \int_0^T \bm{1}_{s \leq t} \Phi_s^{-1} B^\star \lambda_s \xi_s \d s, 
\\
   \sE[(X^\varphi_t)^2 - ( X^\nu_t)^2|\cF_0] &= \left(\Phi_t \int_0^t \Phi_s^{-1} B^\star \lambda_s \xi_s \d s \right)^2 + 2 \sE [X^\nu_t|\cF_0] \Phi_t \int_0^t \Phi_s^{-1} B^\star \lambda_s \xi_s \d s  \\
    &=  \Phi_t^2 (B^\star)^2 \int_0^T  \int_0^T  \bm{1}_{s \leq t} \bm{1}_{u \leq t} \Phi_s^{-1} \Phi_u^{-1} \lambda_s \lambda_u \xi_s \xi_u \d s \d u 
    \\
    &\quad 
     + 2 \left( \Phi_t x_0 + \Phi_t \int_0^t \Phi_s^{-1}  B^\star k_s   \d s
     \right) 
      \Phi_t \int_0^T \bm{1}_{s \leq t} \Phi_s^{-1} B^\star \lambda_s \xi_s \d s.
\end{align*}
Substituting the expressions of $\nu\sim \cG(k, K ,\lambda)$ and  $\varphi\sim \cR(k, K, \lambda,\xi)$ 
into \eqref{eq:cost_RL_relax} and \eqref{eq:cost_randomise}, respectively, and applying Fubini's theorem, we have 
\begin{align}
\label{eq:cost_difference}
\begin{split}
    J^{\theta^\star}(\varphi) - \widetilde{J}^{\theta^\star}_0(\nu) &= \int_0^T \Big(Q_t + 2S_t  K_t + R_t ( K_t)^2 \Big) \sE\big[(X^\varphi_t)^2 - ({X}^\nu_t)^2 \big| \cF_0 \big]\, \d t \\
    &\quad  + \int_0^T 2 \Big( S_t  k_t + R_t  K_t  k_t + p_t + q_t  K_t\Big) \sE\big[X^\varphi_t - {X}^\nu_t \big| \cF_0 \big] \, \d t \\
    &\quad  + \int_0^T R_t \lambda_t^2(\xi_t^2-1)\, \d t + M  \sE\big[(X^\varphi_T)^2 - ({X}^\nu_T)^2 \big| \cF_0 \big] + 2m \sE\big[X^\varphi_T - {X}^\nu_T \big| \cF_0 \big] \\
    &\quad  + \int_0^T 2 ( S_t +   R_t K_t) \sE [X^\varphi_t - X^\nu_t| \cF_0] \lambda_t \xi_t \, \d t + \int_0^T 2(  R_t k_t +  q_t ) \lambda_t \xi_t \,\d t.
\end{split}
\end{align}
Substituting $\sE[ X^\varphi_t - X^\nu_t |\cF_0]$ and $\sE[(X^\varphi_t)^2 - ( X^\nu_t)^2|\cF_0]$
into \eqref{eq:cost_difference} 
 and
 applying Fubini's theorem to 
  reorder the integrations yield  the desired conclusion.
\end{proof}

}

To prove Theorem   \ref{Thm:execution gap_drift_control},
 the following concentration inequalities 
of independent normal random variables will be used,
whose proof is given in Appendix \ref{appendix:techinical_lemma}. 

\begin{Lemma}
 \label{lem:Weak Law for average Z}
 There exists a constant $C \ge 0$ such that 
 for all 
 independent  standard normal random variables 
 $(\zeta_i)_{i\in \sN}$,
 for all $N\in \sN$,
  $(\rho_{i})_{i=1}^N\subset  \sR$,
  $(\beta_{ij})_{i,j=1}^N\subset  \sR$
  and $\delta > 0$,
 \begin{enumerate}[(1)]
     \item 
     \label{item:Weak Law for average Z}
  $
\sP\left(\left|\sum_{i=1}^N \rho_i \zeta_i
\right|
\ge  
C  \|\rho\|_2
\Big( 1 + \sqrt{\ln\big(\tfrac{1}{\delta}} \big) \Big)
\right)
\le \delta$,
with $\|\rho\|_2=\sqrt{\sum_{i=1}^N \rho_i^2}$,
\item
\label{item:Weak Law for sq}
$
\sP\left(
\left|\sum_{i=1}^N \rho_i(\zeta^2_i-1)
\right|
\ge C\|\rho\|_2
\left(1+
\ln \left (\frac{1}{\delta}\right)
\right)
\right)\le \delta$,
with $\|\rho\|_2=\sqrt{\sum_{i=1}^N \rho_i^2}$,
\item
\label{item:Weak Law for cross}
$\sP\left(\left|\sum_{i,j=1,i\not =j}^N \beta_{ij}\zeta_i\zeta_j
\right|
\ge 
C\|{\beta}\|_2
\left(1+
\ln \left (\frac{1}{\delta}\right)
\right)\right)\le \delta$,
with
 $\|\beta\|_2=\sqrt{\sum_{i,j=1,i\not=j}^N \beta_i^2}$.
 \end{enumerate}
 \end{Lemma}

%
 
\begin{proof}[Proof of Theorem \ref{Thm:execution gap_drift_control}]

Throughout this proof, 
let $C$ be a generic constant independent of $\lambda, \xi, \delta$ and $\pi$, but may possibly depend on the sup-norms of the coefficients in (H.\ref{assum:lq_rl})
and the functions
 $k$ and $K$. 
 Observe that the functions $h _i$, $i=1,\ldots, 3$,
 in Lemma \ref{lemma:cost_difference}
 are uniformly bounded by some constant $C\ge 0$.
Let 
$\pi=\{0=t_0<\cdots<t_N=T\}$ be a partition of $[0,T]$ with mesh size $|\pi|$,
and let $\varphi\sim \cR(\mu,\lambda,\xi^\pi)$
with 
$\xi^\pi_t=\sum^{N-1}_{i=0} \zeta_i\bm{1}_{t_i\le t< t_{i+1}}$.
We omit the superscript $\pi$ in $\xi^\pi$ for notational simplicity.

To prove Theorem \ref{Thm:execution gap_drift_control}
Item \ref{item:gap_high_prob},
  it suffices to upper bound all terms  of \eqref{eq:cost_expansion}
in high probability.
The definition of $\xi$ implies that 
$\int_0^T h _{1,t} \lambda_t \xi_t\, \d t = \sum_{i=0}^{N-1} \zeta_i \left(\int_{t_i}^{t_{i+1}} h _{1,t} \lambda_t\, \d t \right)$.
By Lemma \ref{lem:Weak Law for average Z} Item \ref{item:Weak Law for average Z}, it holds  with probability at least $1-\delta$ that
\begin{equation}
\label{eq:h1bound}
\begin{split}
   \Big| \int_0^T h _{1,t} \lambda_t \xi_t \Big|
   &\leq C \Big( 1 + \sqrt{\ln\big(\tfrac{1}{\delta}} \big) \Big) \sqrt{\sum_{i=0}^{N-1} \Big(\int_{t_i}^{t_{i+1}} h _{1,t} \lambda_t \, \d t \Big)^2} 
    \leq
    C\sqrt{T} \sqrt{\pi}  \|\lambda\|_\infty  \Big( 1 + \sqrt{\ln\big(\tfrac{1}{\delta}} \big) \Big).
\end{split}
\end{equation}
Also   applying Lemma \ref{lem:Weak Law for average Z} Item \ref{item:Weak Law for sq} to  $\int_0^T h _{2,t} \lambda_t^2 (\xi_t^2 -1) \, \d t $ implies that with probability at least $1-\delta$, 
\begin{align}
\label{eq:h2bound}
   \Big| \int_0^T h _{2,t} \lambda_t (\xi_t^2-1)\, \d t \Big| 
   &\leq C\sqrt{T} \sqrt{\pi}  \|\lambda\|_\infty 
   \Big( 1 + {\ln\big(\tfrac{1}{\delta}} \big) \Big).
\end{align}
It now remains to quantify the last term of \eqref{eq:cost_expansion}. Observe that
\begin{align*}
   & \int_0^T \int_0^T h _{3,t,r} \lambda_r \lambda_t \xi_r \xi_t  \, \d r \d t
   \\
    &\quad  = \sum_{i,j=0 ; i \neq j}^{N-1}  \zeta_i \zeta_j \left(\int_{t_i}^{t_{i+1}}  \int_{t_j}^{t_{j+1}} h _{3,t,r} \lambda_r \lambda_t  \, \d r \d t \right) 
     + \sum_{i=0}^{N-1}  \zeta_i^2  \left(\int_{t_i}^{t_{i+1}}  \int_{t_i}^{t_{i+1}} h _{3,t,r} \lambda_r \lambda_t  \, \d r \d t \right).
\end{align*}
By
Lemma \ref{lem:Weak Law for average Z} Item \ref{item:Weak Law for cross}, it holds  with probability at least $1-\delta$, 
\begin{align*}
   & \bigg| \sum_{i,j=0 ; i \neq j}^{N-1}  \zeta_i \zeta_j \Big(\int_{t_i}^{t_{i+1}}  \int_{t_j}^{t_{j+1}} h _{3,t,r} \lambda_r \lambda_t  \, \d r \d t \Big) \bigg| 
   \\
   & \quad \leq C  \|\lambda\|^2_\infty 
   \Big( 1 + {\ln\big(\tfrac{1}{\delta}} \big) \Big) \sqrt{\sum_{i,j=0 ; i\neq j}^{N-1}  (t_{i+1}-t_i)^2 (t_{j+1}-t_j)^2}
   \le
C  \|\lambda\|^2_\infty 
   \Big( 1 + {\ln\big(\tfrac{1}{\delta}} \big) \Big) T|\pi|,
\end{align*}
while by Lemma \ref{lem:Weak Law for average Z}
Item \ref{item:Weak Law for sq},
it holds with probability at least $1-\delta$, 
\begin{align*}
   &
   \bigg|  \sum_{i=0}^{N-1}  \zeta_i^2  \Big(\int_{t_i}^{t_{i+1}}  \int_{t_i}^{t_{i+1}} h _{3,t,r} \lambda_r \lambda_t \, \d r \d t \Big) \bigg| 
   \\
   & \le \bigg|  \sum_{i=0}^{N-1}  (\zeta_i^2 -1) \Big(\int_{t_i}^{t_{i+1}}  \int_{t_i}^{t_{i+1}} h _{3,t,r} \lambda_r \lambda_t  \, \d r \d t \Big) \bigg| 
    +\left|
   \sum_{i=0}^{N-1} \Big(\int_{t_i}^{t_{i+1}}  \int_{t_i}^{t_{i+1}} h _{3,t,r} \lambda_r \lambda_t  \, \d r \d t \Big)\right|
   \\
   & \leq C  \|\lambda\|^2_\infty 
   \Big( 1 + {\ln\big(\tfrac{1}{\delta}} \big) \Big) \sqrt{\sum_{i=0}^{N-1}  (t_{i+1}-t_i)^4}
   +C  \|\lambda\|^2_\infty 
  T|\pi|
   \leq C  \|\lambda\|^2_\infty 
   \Big( 1 + {\ln\big(\tfrac{1}{\delta}} \big) \Big)|\pi|.
\end{align*}
Summarising these inequalities
shows  that
with probability at least $1-\delta$,
\begin{equation}
    \label{eq:h3bound}
\begin{split}
   &\Big| \int_0^T \int_0^T h _{3,t,r} \lambda_r \lambda_t \xi_r \xi_t \bm{1}_{r \le t} \, \d r \d t \Big|
   \leq C |\pi| \|\lambda\|^2_\infty  \Big( 1 + {\ln\big(\tfrac{1}{\delta}} \big) \Big) .
\end{split}
\end{equation}
Combining  \eqref{eq:h1bound},
\eqref{eq:h2bound} and
\eqref{eq:h3bound} 
 with the inequality
 that $\sqrt{\ln\big(\tfrac{1}{\delta}\big)}\le 1+\ln\big(\tfrac{1}{\delta}\big)$ for all $\delta\ge 0$
proves  
Theorem \ref{Thm:execution gap_drift_control}
Item \ref{item:gap_high_prob}.

We proceed to prove Theorem \ref{Thm:execution gap_drift_control}
Item \ref{item:gap_as}.
Let $C$ be the   constant   in  Theorem \ref{Thm:execution gap_drift_control} Item \ref{item:gap_high_prob},
and for each $n\in \sN$, 
let $\delta_n=\frac{1}{n^2}$, and 
$A_n = 
\left|J(\varphi^n)-\widetilde{J}(\nu)
\right|
\geq C \sqrt{|\pi_n|} 
 \|\lambda\|_\infty(1+
\|\lambda\|_\infty)
\Big( 1 +  \ln\big(\tfrac{1}{\delta_n}\big)\Big)
$. 
   By Theorem \ref{Thm:execution gap_drift_control} Item \ref{item:gap_high_prob},
$\sum_{n=1}^\infty \sP(A_n)\le \sum_{n=1}^\infty \delta_n<\infty $,
which along with the Borel--Cantelli lemma
shows that 
$\sP(\limsup_{n\to\infty} A_n)=0$.,
Consequently, for $\sP$-a.s. $\om\in \Om$,  
there exists $N_\om\in \sN$, such that 
for all $n\ge N_\om$,
\begin{align*}
\big|J(\varphi^n)-\widetilde{J}(\nu)
\big|
&\le C \sqrt{|\pi_n|} 
 \|\lambda\|_\infty(1+
\|\lambda\|_\infty)
\Big( 1 +  \ln\big(\tfrac{1}{\delta_n}\big)\Big)
\le C n^{-\frac{1}{2}}({\ln n })
 \|\lambda\|_\infty(1+
\|\lambda\|_\infty).
\qedhere
\end{align*}

It remains to   prove Theorem \ref{Thm:execution gap_drift_control}
Item \ref{item:gap_exp}.
By Lemma \ref{lemma:cost_difference} and the fact that 
$\xi_t\sim \cN(0,1)$ for all $t\in [0,T]$
 and $|\sE[\xi_t \xi_s]| \le \bm{1}_{|t-s| \leq |\pi|}$, it follows from \eqref{eq:cost_expansion} that
\begin{align*}
    |\sE[J^{\theta^\star}(\varphi)-\widetilde{J}^{\theta^\star}_0(\nu)]| &=  
    \left|
    \sE\left[
    \int_0^T \int_0^T h _{3,t,r} \lambda_r \lambda_t \xi_r \xi_t \, \d r \d t
    \right]\right|
     \leq C \|\lambda\|^2_\infty \int_0^T \int_0^T |\sE [\xi_r \xi_t] | \d r \d t \\
    &=  C \|\lambda\|^2_\infty \int_0^T \int_{t-|\pi|}^{t+|\pi|}  \d r \d t  \leq C \|\lambda \|_\infty^2 \int_0^T 2|\pi| \d t \leq C |\pi| \|\lambda \|_\infty^2.
\end{align*}    
This proves the desired estimate   in Theorem \ref{Thm:execution gap_drift_control}
Item \ref{item:gap_exp}.
%
%
\end{proof}
%

 \section{Regret analysis for episodic learning}
\label{sec:regret_proof}
\subsection{Analysis of  piecewise constant randomised  policies}
\label{sec:general_policy}

This section 
analyses  
a sequence of 
 piecewise constant randomised  policies
 that is adapted to the observation filtration.
The precise definition of these policies are given  in 
 Setting \ref{setting:randomised_algorithm},
 which 
 includes as special cases 
 the randomised policies generated by 
 Algorithms \ref{alg:exploration} and \ref{alg:MD}. 

\begin{Setting}
\label{setting:randomised_algorithm}
Let  
 ${\theta}_0\in \sR^{1\t 2}$, $V_0\in \sS_+^2$, 
and for each $m\in \sN$, 
let $K^m,k^m\in L^\infty(\Om\t [0,T]; \sR)$,
let 
  $\lambda^m\in L^\infty([0,T];\sR_{\ge 0})$,
  let $\pi_m=(t^m_i)_{i=0}^{N_m}$ be a grid of $[0,T]$,
    let  $\varphi^m:\Om \t [0,T]\t \sR\to \sR$ be defined as in \eqref{eq:randomised_control_m_penalty},
     let      the processes
 $(X^m,Z^m)$ be defined  as in  
 \eqref{eq:LQ_RL_state_m},
 and let $(\hat{\theta}_m,V_m)$ be defined as in 
\eqref{eq: statistics}. 
  Assume that 
$L\coloneqq \sup_{ m\in \sN}( \|K^m\|_{L^\infty}+
\|k^m\|_{L^\infty} + \|\lambda^m\|_{\infty})<\infty$,
 and for all $m\in \sN$, 
$ K^m$ and $k^m$ are $\cF^{\textrm{ob}}_{m-1}\otimes \cB([0,T])$-measurable, with
$\cF^{\textrm{ob}}_{m-1}$  defined by \eqref{eq:observation}.



\end{Setting}


The following proposition establishes a subexponential concentration behaviour of 
$(V^{-1}_{ m})_{m\in \sN}$.

\begin{Proposition}
\label{prop:concentration_G}
Suppose  (H.\ref{assum:lq_rl}) 
and  Setting \ref{setting:randomised_algorithm}
hold.
 Then 
there exists  ${C}\ge 0$,
 depending only on  the coefficients in  (H.\ref{assum:lq_rl})
 and $L$ in Setting \ref{setting:randomised_algorithm}, such that 
 for all $m\in \sN$ and $\delta>0$,
\begin{align*}
&\sP\left(
\left|
V^{-1}_{ m}- V_0^{-1} -
\sum_{n=1}^m\sE\left[
\int_0^T
\frac{1}{\bar{\sigma}^{2}_t} Z^{  n}_t (  Z^{  n}_t)^\top \,\d t
\bigg|
\cF^{\textrm{ob}}_{n-1}
 \right]
 \right|\ge 
 C \max\left(
\ln\left(\tfrac{2}{\delta }\right),
\sqrt{m \ln\left(\tfrac{2}{\delta }\right) }\right)
\right)
\le \delta,
\end{align*}
\end{Proposition}

\begin{proof}
Throughout this proof, let $C$  be a generic constant 
 depending only on  the coefficients in  (H.\ref{assum:lq_rl})
 and $L $ in Setting \ref{setting:randomised_algorithm}.
For each $q\in [1,2]$ and $\sigma$-algebra $\cG$,
let $\|\cdot\|_{q,\cG}:\Om\to [0,\infty]$ be  
the conditional Orlicz norm defined in \cite[Definition 4.1]{szpruch2021exploration}, such that for all random variables $X:\Om\to \sR$,
$$
\|X\|_{q,\cG}
\coloneqq  \mathcal{G}\text{-}\essinf \big\{ Y \in L^0(\cG ; (0, \infty))\; \big| \; 
\sE\left[\exp (|X|^q/Y^q) \big| \cG\right]\le 2\big\},
$$
where  $L^0(\cG ; (0, \infty))$ is the set of all
$\cG/\cB((0,\infty))$-measurable functions
$Y:\Om\to (0,\infty)$.

We first prove that there exists a constant $C\ge 0$ such that  $\|(\int_0^T |X^m_t|^2\,\d t)^{1/2}\|_{2,\cF^{\textrm{ob}}_{m-1}}\le C$ for all $m\in \sN$.
Observe from \eqref{eq:LQ_RL_state_m} that, 
for each $m\in \sN$, 
$X^m=\cX^m+\cY^m$, where 
\begin{align}
    \label{eq:state_decomposiition}
\begin{split}
    \d \cX^m_t &= 
    \left(A^\star \cX^m_t +B^\star  (K^m_t \cX^m_t+k^m_t) \right)\, \d t
     +  
    \bar{\sigma}_t\,
     \d W^m_t,
     \quad  t\in [0,T]; \quad \cX^m_0 = x_0,
     \\
    \d \cY^m_t &= 
    \left(A^\star \cY^m_t +B^\star  (K^m_t \cY^m_t+\lambda^m_t \xi^{\pi_m}_t) \right)\, \d t,
     \quad  t\in [0,T]; \quad \cY^m_0 = 0.
     \end{split}
\end{align}
By
 \cite[Proposition 4.7]{szpruch2021exploration},
 $\|(\int_0^T |\cX^m_t|^2\,\d t)^{1/2}\|_{2,\cF^{\textrm{ob}}_{m-1}}\le C$ for all $m\in \sN$.
 Moreover,
let 
 $\Phi^m:\Om\t [0,T]\to \sR$ be such that
for all $(\om,t)\in \Om\t [0,T]$,
$\d \Phi^m_t(\om)=(A^\star+B^\star K^m_t(\om))\Phi^m_t(\om) \,\d t$
and $\Phi^m_0(\om)=1$,
then
$\cY^{ m}_t
=
\Phi^m_t  \int_0^t(\Phi^m_s)^{-1}
B^\star \lambda^m_s \xi_s^{\pi_m}\,\d s
$ for all $t\in [0,T]$.
Observe that $\Phi^m$ is invertible $\d\sP\otimes \d t$ a.e., and by 
the uniform boundedness of $(K^m)_{m\in \sN}$,
 $\|\Phi^m\|_{L^\infty}+\|(\Phi^m)^{-1}\|_{L^\infty}\le C$ for all $m\in \sN$.
Hence, by H\"{o}lder's inequality,
$$
\left(\int_0^T |\cY^m_t|^2\,\d t\right)^{\frac{1}{2}}
\le C\sup_{t\in [0,T]} |\cY^m_t|
\le C\sup_{t\in [0,T]} \int_0^t |\lambda^m_t \xi^{\pi_m}_s|\, \d s
\le C \|\lambda^m\|_\infty  \left(\int_0^T |\xi^{\pi_m}_s|^2\, \d s\right)^{\frac{1}{2}},
$$
Assume without loss of generality that 
$\|\lambda^m\|_\infty>0$.
By the convexity of 
 $x\mapsto e^x$ 
and Jensen's inequality, 
for a sufficiently large $\tilde{C}>0$,
\begin{align*}
&\sE\left[\exp \left(\frac{1}{\tilde{C}^2\|\lambda^m\|^2_\infty}\int_0^T |\cY^m_t|^2\,\d t\right) \bigg| \cF^{\textrm{ob}}_{m-1}\right]
\le
\sE\left[\exp \left(\int_0^T \frac{C^2}{\tilde{C}^2}|\xi^{\pi_m}
_t|^2\,\d t\right)\right]
\\
&\quad \le \frac{1}{T}
\sE\left[\int_0^T\exp \left( \frac{TC^2}{\tilde{C}^2}|\xi^{\pi_m}
_t|^2\right)\,\d t\right]
\le
\sE\left[\exp \left( \frac{TC^2}{\tilde{C}^2}|\xi^{\pi_m}
_0|^2\right)\right]\le 2,
 \end{align*}
 where we used that $(\xi^{\pi_m}_t)_{t\in [0,T]}$ are standard normal random variables independent of $\cF^{\textrm{ob}}_{m-1}$.
Consequently, $\|(\int_0^T |\cY^m_t|^2\,\d t)^{\frac{1}{2}}\|_{2,\cF^{\textrm{ob}}_{m-1}}\le \tilde{C}$, which implies that $\|(\int_0^T |X^m_t|^2\,\d t)^{1/2}\|_{2,\cF^{\textrm{ob}}_{m-1}}\le C$ for all $m\in \sN$.

Observe that for all $m\in \sN$ and $t\in [0,T]$,
$\varphi^m(\cdot, t, X^{m}_t)=K^m_t X^{m}_t+k^m_t+\lambda^m_t \xi^{\pi_m}_t$.
By the uniform boundedness of $K^m$ and $k^m$,
 $\|(\int_0^T |\varphi^m(\cdot, t, X^{m}_t)|^2\,\d t)^{\frac{1}{2}}\|_{2,\cF^{\textrm{ob}}_{m-1}}\le {C}$ for all $m\in \sN$,
which along with $\bar{\sigma}^{-1}\in L^\infty([0,T];\sR)$ and 
\cite[Proposition 4.3]{szpruch2021exploration} implies that 
\begin{align*}
&\left\|
\int_0^T
\frac{1}{\bar{\sigma}^{2}_s} Z^{ m}_s (  Z^{ m}_s)^\top \,\d s
\right\|_{1,\cF^{\textrm{ob}}_{m-1}}
\le 
C
\left\|
\int_0^T
\begin{pmatrix}
X^{ m}_s
\\
\varphi^{m}(\cdot, s,X^{ m}_s)
\end{pmatrix}
\begin{pmatrix}
X^{ m}_s
\\
\varphi^m(\cdot, s,X^{ m}_s)
\end{pmatrix}^\top \,\d s
\right\|_{1,\cF^{\textrm{ob}}_{m-1}}
\\
&\qquad 
\le 
C
\bigg(1+\Big\|\Big(\int_0^T |X^m_t|^2\,\d t\Big)^{\frac{1}{2}}\Big\|_{2,\cF^{\textrm{ob}}_{m-1}}^2 
+\Big\|\Big(\int_0^T |\varphi^m(\cdot, t, X^{m}_t)|^2\,\d t\Big)^{\frac{1}{2}}\Big\|^2_{2,\cF^{\textrm{ob}}_{m-1}}
\bigg)\le 
{C}.
\end{align*}
As 
$V^{-1}_m
=
V_0^{-1}+
\sum_{n=1}^m
\int_0^T
\frac{1}{\bar{\sigma}^{2}_s} Z^{  n}_s (  Z^{ n}_s)^\top \,\d s$ and 
$\int_0^T
\frac{1}{\bar{\sigma}^{2}_s} Z^{m}_s (  Z^{ m}_s)^\top \,\d s
$ is $\cF^{\textrm{ob}}_{m}$-measurable,
by 
the  Bernstein inequality for martingale difference sequence in \cite[Proposition 4.4]{szpruch2021exploration},
 for all $m\in \sN$ and $\eps>0$,
\begin{align*}
&\sP\left(
\left|
V^{-1}_{ m}- V_0^{-1}-
\sum_{n=1}^m\sE\left[
\int_0^T
\frac{1}{\bar{\sigma}^{2}_s} Z^{ n}_s (  Z^{  n}_s)^\top \,\d s
\bigg|
\cF^{\textrm{ob}}_{n-1}
 \right]
 \right|\ge m\eps
\right)
\le 2\exp\left(
-mC\min \left(\eps^2, \eps \right)
\right).
\end{align*}
The desired result follows by substituting $\eps$ such that $2\exp\left(
-mC\min \left(\eps^2, \eps \right)
\right)=\delta$.
\end{proof}

We  now establish  upper and lower bounds of $\sum_{n=1}^m\sE\big[
\int_0^T
\frac{1}{\bar{\sigma}^{2}_s} Z^{ n}_s (  Z^{  n}_s)^\top \,\d s
\big|
\cF^{\textrm{ob}}_{n-1}
 \big]
$
 in terms of $m$.

\begin{Proposition}
\label{prop:upper_lower_expectation}
Suppose (H.\ref{assum:lq_rl}) 
and Setting \ref{setting:randomised_algorithm}
hold.
Then there exists a constant $C>0$, 
 depending only on  the coefficients in  (H.\ref{assum:lq_rl})
and  $L$ in Setting \ref{setting:randomised_algorithm}, such that 
\begin{enumerate}[(1)]
\item\label{eq:G_conditonal_upper}
for all $m\in \sN$,
$\sum_{n=1}^m\left\|\sE\left[
\int_0^T
\frac{1}{\bar{\sigma}^{2}_t} Z^{ n}_t (  Z^{  n}_t)^\top \,\d t
\big|
\cF^{\textrm{ob}}_{n-1}
 \right]\right\|\le Cm$.
\item
\label{eq:G_conditonal_lower}
if $|\pi_{m}| \|\lambda^m\|_\infty^2\le \frac{1}{C}
\min_{t\in [0,T]}((\lambda^m_t)^2 \wedge 1)$ for all $m\in \sN$, then
$$
\sum_{n=1}^m\Lambda_{\min}\left(
\sE\left[
\int_0^T
\frac{1}{\bar{\sigma}^{2}_t} Z^{  n}_t (  Z^{ n}_t)^\top \,\d t
\bigg|
\cF^{\textrm{ob}}_{n-1}
 \right]\right)
 \ge \frac{1}{C}\sum_{n=1}^m \min_{t\in [0,T]}((\lambda^n_t)^2 \wedge 1),
 \quad \fa m\in \sN,
$$
where $a\wedge 1=\min (a,1)$ for all $a\in \sR$.
\end{enumerate}
\end{Proposition}

\begin{proof}
For all $m\in \sN$, 
by the definition of $Z^m$ in \eqref{eq:LQ_RL_state_m},
the measurability conditions of $K^m$ and $k^m$ and the fact that $\sE\big[\xi^{\pi_m}_t \big| \cF^{\textrm{ob}}_{m-1} \big] = 0$ for all $t \in [0,T]$  implies that 
for $\d\sP$-a.s.,
\begin{align}
&\sE\left[
\int_0^T
\frac{1}{\bar{\sigma}^{2}_t} Z^{ m}_t (  Z^{ m}_t)^\top \,\d t
\bigg|
\cF^{\textrm{ob}}_{m-1}
\right](\om) 
=\sE\left[
\int_0^T
\frac{1}{\bar{\sigma}^{2}_t} 
\begin{psmallmatrix}
 {X}^{ m}_t
\\
K^m_t  {X}^{ m}_t+k^m_t +\lambda^m_t \xi_t^{\pi_m}
\end{psmallmatrix}
\begin{psmallmatrix}
 {X}^{ m}_t
\\
K^m_t  {X}^{ m}_t+k^m_t +\lambda^m_t \xi_t^{\pi_m}
\end{psmallmatrix}^\top 
\,\d t
\bigg|
\cF^{\textrm{ob}}_{m-1}
\right](\om)
\nb
\\
&
\quad 
=  \sE\bigg[
\int_0^T
\frac{1}{\bar{\sigma}^{2}_t}
\left(
\begin{psmallmatrix}
0 & 0
\\
 \widetilde{X}^{ m, \omega}_t
\lambda^m_t \xi^{\pi_m}_t & K^m_t(\om) \widetilde{X}^{ m, \omega}_t
\lambda^m_t \xi^{\pi_m}_t
\end{psmallmatrix}
+\begin{psmallmatrix}
0 &  \widetilde{X}^{ m, \omega}_t \lambda^m_t \xi^{\pi_m}_t 
\\
0 & K^m_t(\om) \widetilde{X}^{ m, \omega}_t
\lambda^m_t \xi^{\pi_m}_t
\end{psmallmatrix}
\right)
\,\d t
\bigg]
\label{eq:conditional_G_m_cross}
\\
& \quad \quad +
\sE\left[
\int_0^T
\frac{1}{\bar{\sigma}^{2}_t} \left(
\begin{psmallmatrix}
\widetilde{X}^{ m, \omega}_t
\\
K^m_t(\om) \widetilde{X}^{ m, \omega}_t+k^m_t(\om)
\end{psmallmatrix}
\begin{psmallmatrix}
\widetilde{X}^{ m, \omega}_t
\\
K^m_t(\om) \widetilde{X}^{ m, \omega}_t+k^m_t(\om)
\end{psmallmatrix}^\top 
+ \begin{psmallmatrix}
0 & 0
\\
0 & (\lambda^m_t)^2 
\end{psmallmatrix} \right) \,\d t
\right]
\label{eq:conditional_G_m_square}
\end{align}
where  $\widetilde{X}^{m, \omega}$ is the state process ${X}^m$ 
conditioned on $\cF^{\textrm{ob}}_{m-1}
$, such that $\sP$-a.s.~$\om$ and for all $t\in [0,T]$,
\begin{align}
\label{eq:X_m_condition}
\widetilde{X}^{ m,\omega}_t(\cdot)
&=
\Phi^m_t(\om) x_0+
\Phi^m_t(\om)  \int_0^t(\Phi^m_s(\om))^{-1}
\Big(
B^\star
(k^m_s(\om)+\lambda^m_s \xi_s^{\pi_m}(\cdot))\,\d s
+\bar{\sigma}_s \, \d W^m_s(\cdot)
\Big), 
\end{align}
with 
$\d \Phi^m_t(\om)=(A^\star+B^\star K^m_t(\om))\Phi^m_t(\om) \,\d t$
and $\Phi^m_0(\om)=1$.
By the uniform boundedness of $K^m$, $k^m$ and $\lambda^m$,
$\sE\left[\sup_{t\in [0,T]}|\widetilde{X}^{m, \omega}_t|^2\right]\le C$ for all $m
\in \sN$.
Hence, 
applying the Cauchy--Schwarz inequality to
\eqref{eq:conditional_G_m_cross}-\eqref{eq:conditional_G_m_square}
 and using the Gaussianity of $(\xi^{\pi_m}_t)_{t\in [0,T]}$ 
yield Item \ref{eq:G_conditonal_upper}.

To prove Item \ref{eq:G_conditonal_lower},
we first establish a lower bound of \eqref{eq:conditional_G_m_cross}. By \eqref{eq:X_m_condition} and the independence between $\xi^{\pi_m}$ and $W^m$,
for all $t\in [0,T]$,
\begin{align*}
\left|
\sE[ \widetilde{X}^{ m, \omega}_t
\lambda^m_t \xi^{\pi_m}_t]
\right|
&=
\left|
 \lambda^m_t 
\Phi^m_t(\om)  \int_0^t(\Phi^m_s(\om))^{-1}
\lambda^m_s 
\sE\left[
\xi_s^{\pi_m}\xi_t^{\pi_m}
\right]\,\d s
\right|.
\\
\left|
\sE[K^m_t(\om) \widetilde{X}^{ m, \omega}_t
\lambda^m_t \xi^{\pi_m}_t]
\right|
&=
\left|
K^m_t(\om) \lambda^m_t 
\Phi^m_t(\om)  \int_0^t(\Phi^m_s(\om))^{-1}
\lambda^m_s 
\sE\left[
\xi_s^{\pi_m}\xi_t^{\pi_m}
\right]\,\d s
\right|.
\end{align*}
Note that  $\sE[\xi_s^{\pi_m}\xi_t^{\pi_m}]$ is $1$
 if there exists $i_0$ such that $s,t\in [t_{i_0}^m,t_{i_0+1}^m)$, and is zero otherwise. By the uniform boundedness of $K^m$,
 there exists $C\ge 0$ such that 
for all $ m\in \sN$,
\begin{equation}
    \label{eq:lower bound eigen}
    \Lambda_{\min}\left(
\sE\bigg[
\int_0^T
\frac{1}{\bar{\sigma}^{2}_t}
\left(
\begin{psmallmatrix}
0 & 0
\\
 \widetilde{X}^{ m, \omega}_t
\lambda^m_t \xi^{\pi_m}_t & K^m_t(\om) \widetilde{X}^{ m, \omega}_t
\lambda^m_t \xi^{\pi_m}_t
\end{psmallmatrix}
+\begin{psmallmatrix}
0 &  \widetilde{X}^{ m, \omega}_t \lambda^m_t \xi^{\pi_m}_t 
\\
0 & K^m_t(\om) \widetilde{X}^{ m, \omega}_t
\lambda^m_t \xi^{\pi_m}_t
\end{psmallmatrix}
\right)
\,\d t
\bigg]
\right)\ge -C |\pi_{m}| \|\lambda^m\|_\infty^2.
\end{equation}
We then prove that the minimum eigenvalue of  \eqref{eq:conditional_G_m_square} is bounded away from zero.
Observe that 
for each   $K \in \sR$,
a direct computation shows that 
all eigenvalues of 
the matrix
$\Psi(K) := \begin{psmallmatrix}
1 \\ K
\end{psmallmatrix}\begin{psmallmatrix}
1 \\ K
\end{psmallmatrix}^\top + \begin{psmallmatrix}
0 & 0 \\ 0& 1
\end{psmallmatrix}
$
are positive, which along with the continuity of the map $\sS^2\ni S\mapsto \Lambda_{\min }(S)\in \sR$ 
implies that 
$c_0\coloneqq \inf_{|K|\le L} \Lambda_{\min} \big( \Psi(K) \big)>0$,
which the constant $L$ in 
  Setting \ref{setting:randomised_algorithm}.
Then for all $m\in \sN$ and $t\in [0,T]$,
\begin{equation}
 \label{eq:min eigen fix time}
\begin{aligned}
    &\Lambda_{\min} \left( \sE\left[
\frac{1}{\bar{\sigma}^{2}_t} \left(
\begin{psmallmatrix}
\widetilde{X}^{ m, \omega}_t
\\
K^m_t(\om) \widetilde{X}^{ m, \omega}_t+k^m_t(\om)
\end{psmallmatrix}
\begin{psmallmatrix}
\widetilde{X}^{ m, \omega}_t
\\
K^m_t(\om) \widetilde{X}^{ m, \omega}_t+k^m_t(\om)
\end{psmallmatrix}^\top 
+ \begin{psmallmatrix}
0 & 0
\\
0 & (\lambda^m_t)^2 
\end{psmallmatrix} \right) 
\right] \right) \\
&\qquad \geq \frac{1}{\bar{\sigma}^{2}_t}  \Lambda_{\min} \left( \text{Var} \left[ \begin{psmallmatrix}
\widetilde{X}^{ m, \omega}_t
\\
K^m_t(\om) \widetilde{X}^{ m, \omega}_t+k^m_t(\om)
\end{psmallmatrix} \right] + \begin{psmallmatrix}
0 & 0
\\
0 & (\lambda^m_t)^2 
\end{psmallmatrix} \right)
\\&\qquad
\geq \frac{1}{\bar{\sigma}^{2}_t}  \Lambda_{\min}
\left(
 \begin{psmallmatrix}
1
\\
K^m_t(\om)
\end{psmallmatrix}   \text{Var} \big[ 
\widetilde{X}^{ m, \omega}_t
 \big] \begin{psmallmatrix}
1
\\
K^m_t(\om)
\end{psmallmatrix}^\top + \begin{psmallmatrix}
0 & 0
\\
0 & (\lambda^m_t)^2 
\end{psmallmatrix} \right)
\\&\qquad
\geq \frac{1}{\bar{\sigma}^{2}_t} \Big( \text{Var} \big[ 
\widetilde{X}^{ m, \omega}_t
 \big] \wedge (\lambda^m_t)^2\Big) \Lambda_{\min} \left(  \begin{psmallmatrix}
1
\\
K^m_t(\om)
\end{psmallmatrix} \begin{psmallmatrix}
1
\\
K^m_t(\om)
\end{psmallmatrix}^\top + \begin{psmallmatrix}
0 & 0
\\
0 & 1
\end{psmallmatrix} \right) \geq c_0  \Big( \text{Var} \big[ 
\widetilde{X}^{ m, \omega}_t
 \big] \wedge (\lambda^m_t)^2\Big),
\end{aligned}
\end{equation}
as $|K_t^m(\om)|\le L$  for 
$\sP$-a.s.
Since  $\inf_{t\in [0,T]}\bar{\sigma}_t>0$, $K^m$ is uniformly bounded
and $\xi^{\pi_m}$ and $W^m$ are independent, 
by 
\eqref{eq:X_m_condition},
\begin{align}
\label{eq:var lower bound}
\begin{split}
\inf_{m\in\sN,t\in [T/2,T]}
\text{Var} \big[ 
\widetilde{X}^{ m, \omega}_t
 \big]
&\ge 
\inf_{m\in\sN,t\in [T/2,T]}
\sE\left[
\left(
\Phi^m_t(\om)  \int_0^t(\Phi^m_s(\om))^{-1}
\Big(
\lambda^m_s \xi_s^{\pi_m}(\cdot))\,\d s
+\bar{\sigma}_s \, \d W^m_s(\cdot)
\right)^2
\right]
\\
  &\ge 
\inf_{m\in\sN, t\in [T/2,T]}
\left(
(\Phi^m_t(\om))^2  \int_0^t\left((\Phi^m_s(\om))^{-1}
\bar{\sigma}_s \right)^2 \,\d s
\right) >0.
\end{split}
\end{align}
Consequently, 
by using   Fubini's theorem, the concavity of $\sS^2\ni S\mapsto \Lambda_{\min} (S)\in \sR$ and Jensen's inequality, 
\begin{equation}
    \label{eq:second term lower bound}
    \begin{aligned}
    &\Lambda_{\min} \left( \sE\left[\int_0^T
\frac{1}{\bar{\sigma}^{2}_t} \left(
\begin{psmallmatrix}
\widetilde{X}^{ m, \omega}_t
\\
K^m_t(\om) \widetilde{X}^{ m, \omega}_t+k^m_t(\om)
\end{psmallmatrix}
\begin{psmallmatrix}
\widetilde{X}^{ m, \omega}_t
\\
K^m_t(\om) \widetilde{X}^{ m, \omega}_t+k^m_t(\om)
\end{psmallmatrix}^\top 
+ \begin{psmallmatrix}
0 & 0
\\
0 & (\lambda^m_t)^2 
\end{psmallmatrix} \right) 
\; \d t\right] \right)
 \\
   &\qquad \geq \Lambda_{\min} \left( \sE\left[\int_{T/2}^T
\frac{1}{\bar{\sigma}^{2}_t} \left(
\begin{psmallmatrix}
\widetilde{X}^{ m, \omega}_t
\\
K^m_t(\om) \widetilde{X}^{ m, \omega}_t+k^m_t(\om)
\end{psmallmatrix}
\begin{psmallmatrix}
\widetilde{X}^{ m, \omega}_t
\\
K^m_t(\om) \widetilde{X}^{ m, \omega}_t+k^m_t(\om)
\end{psmallmatrix}^\top 
+ \begin{psmallmatrix}
0 & 0
\\
0 & (\lambda^m_t)^2 
\end{psmallmatrix} \right) 
\; \d t\right] \right)
  \\ &\qquad
  \geq \int_{T/2}^T \Lambda_{\min} \left( \sE\left[
\frac{1}{\bar{\sigma}^{2}_t} \left(
\begin{psmallmatrix}
\widetilde{X}^{ m, \omega}_t
\\
K^m_t(\om) \widetilde{X}^{ m, \omega}_t+k^m_t(\om)
\end{psmallmatrix}
\begin{psmallmatrix}
\widetilde{X}^{ m, \omega}_t
\\
K^m_t(\om) \widetilde{X}^{ m, \omega}_t+k^m_t(\om)
\end{psmallmatrix}^\top 
+ \begin{psmallmatrix}
0 & 0
\\
0 & (\lambda^m_t)^2 
\end{psmallmatrix} \right) 
\right] \right) \; \d t \\
&\qquad
\geq \tilde{c}_0 \min_{t \in [0,T]} \Big( 1 \wedge (\lambda^m_t)^2\Big),
    \end{aligned}
\end{equation}
for some $\tilde{c}_0>0$,
where the final inequality follows from \eqref{eq:min eigen fix time} and \eqref{eq:var lower bound}. 
Hence 
setting
$|\pi_{m}| \|\lambda^m\|_\infty^2\le \frac{\tilde{c}_0}{2C}
\min_{t\in [0,T]}((\lambda^m_t)^2 \wedge 1)$
with $C$ given in \eqref{eq:lower bound eigen} 
and using 
 \eqref{eq:conditional_G_m_cross}
\eqref{eq:conditional_G_m_square}, \eqref{eq:lower bound eigen} and \eqref{eq:second term lower bound}  
prove Item  \ref{eq:G_conditonal_lower}.
\end{proof}

Based on Propositions \ref{prop:concentration_G} and    \ref{prop:upper_lower_expectation},
the following lemma estimates   the accuracy of  $(\hat{\theta}_m)_{m\in \sN}$   in terms of  $(V^{-1}_m)_{m\in \sN}$,
which extends 
 \cite[Proposition 4.6]{szpruch2021exploration}
 to randomised policies.
 

\begin{Lemma}
\label{lemma:theta_accuracy}
Suppose  (H.\ref{assum:lq_rl}) 
and Setting \ref{setting:randomised_algorithm}
hold.
 Then 
there exists a constant ${C}\ge 0$,
 depending only on  the coefficients in  (H.\ref{assum:lq_rl})
 and $\theta_0,V_0, L$ in Setting \ref{setting:randomised_algorithm},
such that for all 
 $m \in \sN$, 
 \begin{equation}
     \label{eq: high prob parameter}
      \sP\Big( | \hat{ {\theta}}_{ m} - {\theta}^\star |^2 \leq C \big(\Lambda_{\min}(V_{ m}^{-1}) \big)^{-1} \big(1 + \ln m \big) \Big) \geq 1-1/m.
 \end{equation}
Assume further that 
  (H.\ref{eq:bounded}) holds.
   Then for all $m\in \sN$,
 \begin{equation}
     \label{eq:expected parameter}
     \sE\big[ | { {\theta}}_{ m} - {\theta}^\star |^2 \big] \leq C \left( \sE \big[\big(\Lambda_{\min}(V_{ m}^{-1}) \big)^{-1}\big] (1 + \ln m) + \frac{1}{m} + \sP(\hat{ {\theta}}_{ m} \notin \Theta) \right),
 \end{equation}
with $\theta_m$   defined by   \eqref{eq:theta_m_project}.
\end{Lemma}

\begin{proof}
Throughout this proof, 
let $C>0$ be a  generic constant
which is 
independent of $\delta\in (0,1)$ and $m\in \sN$ 
and may take   different values at each occurrence.

By
the  same arguments for
Lemma 4.5 and Proposition 4.6 in \cite{szpruch2021exploration} 
and the fact that  $\inf_{t\in [0,T]}\bar{\sigma}_t
>0$, it follows that for all 
$\delta \in (0,1)$ and $m \in \sN$, 
\begin{equation}
    \label{eq:high prob bound}
    \sP\Big(\Lambda_{\min}(V_{ m}^{-1}) | \hat{ {\theta}}_{ m} - {\theta}^\star |^2 \leq C \big(1 + \ln ( \det V_{ m}^{-1} ) + \ln(\tfrac{1}{\delta}) \big) \Big) \geq 1-\delta.
\end{equation}

By Proposition \ref{prop:concentration_G} and Proposition  \ref{prop:upper_lower_expectation} Item \ref{eq:G_conditonal_upper}, it follows that  for all 
$\delta \in (0,1)$ and $m \in \sN$, with probability at least $1-\delta$,
\begin{align}
\label{eq:det bound}
\ln ( \det V_{ m}^{-1} ) &
\leq C \ln |V_m^{-1}| 
\le C\ln \left( \left| V_0^{-1}+\sum_{n=1}^m\sE\left[
\int_0^T
\frac{1}{\bar{\sigma}^{2}_t} Z^{  n}_t (  Z^{  n}_t)^\top \,\d t
\bigg|
\cF^{\textrm{ob}}_{n-1}
 \right] \right| + C \max\left(
\ln\left(\tfrac{2}{\delta }\right),
\sqrt{m \ln\left(\tfrac{2}{\delta }\right) }\right) \right) \nonumber\\
&\leq C \ln \left(Cm + C \max\Big(
\ln\left(\tfrac{2}{\delta }\right),
\sqrt{m \ln\left(\tfrac{2}{\delta }\right) }\Big) \right) \leq C\Big(1 + \ln m + \ln\left(\tfrac{1}{\delta }\right)\Big).
\end{align}
Substituting $\delta = 1/(2m)$ in \eqref{eq:high prob bound} and \eqref{eq:det bound} implies that
$$ \sP\Big( | \hat{ {\theta}}_{ m} - {\theta}^\star |^2 \leq C \big(\Lambda_{\min}(V_{ m}^{-1}) \big)^{-1} \big(1 + \ln m \big) \Big) \geq 1-1/m.$$
Let $A$ be the event such that the estimate 
$| \hat{ {\theta}}_{ m} - {\theta}^\star |^2 \leq C \big(\Lambda_{\min}(V_{ m}^{-1}) \big)^{-1} \big(1 + \ln m \big)$
 holds. 
 Under (H.\ref{eq:bounded}), 
 $\theta_m$ and $\theta^\star$ are uniformly bounded
 (cf.~\eqref{eq:theta_m_project}). Hence, 
\begin{align*}
    \sE \big[| { {\theta}}_{ m} - {\theta}^\star |^2 \big] &\leq \sE \big[| { {\theta}}_{ m} - {\theta}^\star |^2 \bm{1}_{A \cap \{\hat{\theta}_m \in \Theta \}}\big] + C \big(\sP(A^c) + \sP(\hat{ {\theta}}_{ m} \notin \Theta) \big) \\ 
    &\leq C \left( \sE \big[\big(\Lambda_{\min}(V_{ m}^{-1}) \big)^{-1}\big] (1 + \ln m) + \frac{1}{m} + \sP(\hat{ {\theta}}_{ m} \notin \Theta) \right)
\end{align*}
where the last inequality uses   the fact that $\theta_m = \hat{\theta}_m$ on the event $\{\hat{\theta}_m \in \Theta\}$.
\end{proof}

The remaining part of  this section
is devoted to analyse 
 the expected costs of  $ (\varphi_m)_{m\in \sN}$.
The following lemma proves 
stability of \eqref{eq:riccati_theta} and  regularity of the cost function $J^{\theta^\star}$ in \eqref{eq:LQ_RL_cost}.
The proof 
 adapts the arguments of
Lemma 3.1 and Proposition 3.7
in 
\cite{basei2022logarithmic}
to the present setting,
and  is given in Appendix  \ref{appendix:techinical_lemma}.

\begin{Lemma}
\label{lemma:regularity}
Suppose (H.\ref{assum:lq_rl}) holds. 
\begin{enumerate}[(1)]
\item\label{item:continuous_differentiability}
The map 
 $\sR^{1\t 2}\ni\theta\mapsto P^\theta,\eta^\theta\in C^1([0,T];\sR)$ is continuously differentiable,
 where   $P^{ \theta}\in C([0,T];\sR_{\ge 0})$ and $\eta^{ \theta}\in C([0,T]; \sR)$
are the solution to
\eqref{eq:riccati_theta}.
\item \label{item:quadratic_dependence}
For each $L\ge 0$, there exists a constant $C\ge 0$,
 depending only on $L$ and  the coefficients in (H.\ref{assum:lq_rl}),
 such that 
for all $K,k\in C([0,T];\sR)$ with $\|K\|_{\infty},\|k\|_{\infty}\in [0, L]$,
\begin{equation*}
     J^{  \theta^\star}(\phi)-     J^{  \theta^\star}( \phi^{ \theta^\star} )\le C(\|K-K^{\theta^\star}\|^2_{\infty}+\|k-k^{\theta^\star}\|^2_{\infty}),
\end{equation*}
where $\phi:[0,T]\t \sR\to \sR$  satisfies   $\phi(t,x)=K_tx+k_t$ 
for all $(t,x)\in [0,T]\t \sR$, 
and $ \phi^{ \theta^\star}$ is the optimal policy of  \eqref{eq:LQ_RL_cost}.

\end{enumerate}

\end{Lemma}

Based on Lemma 
\ref{lemma:regularity},
the next proposition   quantifies  the performance gaps between 
piecewise constant randomised  policies $(\varphi_m)_{m\in \sN}$
    and the optimal policy $\phi^{\theta^\star}$.

\begin{Proposition}
\label{prop: sensitivity}
Suppose (H.\ref{assum:lq_rl}) 
and Setting \ref{setting:randomised_algorithm}
hold.
 Then 
there exists a constant ${C}\ge 0$,
 depending only on  the coefficients in  (H.\ref{assum:lq_rl})
 and $L$ in Setting \ref{setting:randomised_algorithm}, such that for all $m \in \sN$,
\begin{align*}
\sE \big[ J^{{\theta^\star}}(\varphi^m)-  J^{{\theta^\star}}(
      \phi^{ \theta^\star}) \big]
&\le
 C
\bigg(
(1+
 {|\pi_m|} )
 \|\lambda^m\|_\infty^2
+
\sE \big[\|K^m-K^{\theta^\star}\|^2_{\infty} \big]+\sE \big[\|k^m-k^{\theta^\star}\|^2_{\infty} \big]
\bigg).
\end{align*}
\end{Proposition}

\begin{proof}
We first  prove that there exists a constant $C\ge 0$ such that 
for all $K,k\in L^\infty([0,T];\sR)$ with $\|K\|_{\infty},\|k\|_{\infty}\in [0, L]$ and $\lambda\in L^\infty([0,T];\sR_{\ge 0})$,
it holds with $\nu(t,x)\coloneqq \cN(K_t x+k_t,\lambda^2_t)$ for all $(t,x)\in [0,T]\t \sR$,
\begin{align}
\label{eq:gaussian_relax_optimal}
|\widetilde{J}^{{\theta^\star}}_0(
     \nu)
     -
      J^{{\theta^\star}}(
      \phi^{ \theta^\star})|
      \le 
      C\left(
\|\lambda\|_\infty^2+
\|K-K^{\theta^\star}\|^2_{\infty}+\|k-k^{\theta^\star}\|^2_{\infty}
      \right),
\end{align}
with $\widetilde{J}^{{\theta^\star}}_0(
     \nu)$ defined as in 
 \eqref{eq:cost_RL_relax}.
Indeed, let $\phi(t,x)\coloneqq K_t x+k_t$
 for all $(t,x)\in [0,T]\t \sR$. Then 
\begin{align}
\label{eq:decomposition_exploration}
\begin{split}
&\widetilde{J}^{{\theta^\star}}_0(
     \nu)
     -
      J^{{\theta^\star}}(
      \phi^{ \theta^\star})
      =
         \left(
\widetilde{J}^{{\theta^\star}}_0(
     \nu)
     -{J}^{{\theta^\star}}(
     \phi)     
          \right)
          +
                   \left(
{J}^{{\theta^\star}}(
     \phi)
     -  J^{{\theta^\star}}(
      \phi^{ \theta^\star})     
          \right).   
\end{split}
\end{align}
For the first term of \eqref{eq:decomposition_exploration},
observe that 
$X^{\theta^\star,\phi}$ and $X^{\theta^\star,\nu}$  satisfy the same dynamics:
 \begin{equation*}
    \d X_t = 
    (A^\star X_t +B^\star (K_t X_t+k_t))\, \d t
     +  
    \bar{\sigma}_t\,
     \d W_t,
     \quad  t\in [0,T]; \quad X_0 = x_0,
\end{equation*}
and hence  $X^{\theta^\star,\phi}=X^{\theta^\star,\nu}$
due to  the uniqueness of solutions.
Thus, by  \eqref{eq:cost_RL_relax},
\begin{align}
\label{eq:exploration_cost}
    \widetilde{J}^{\theta^\star}_0(\nu) 
    &= \sE\Bigg[
    \int_0^T \Big( f(t,X^{\theta^\star,\phi}_t,   K_tX^{\theta^\star,\phi}_t+k_t) + R_t \lambda_t^2 \Big)\, \d t
    +g(X^{\theta^\star,\phi}_T) \bigg]
=  J^{  \theta^\star}(\phi)+\int_0^T R_t \lambda_t^2\,\d t,
\end{align}
which along with   $R\in L^\infty([0,T];\sR)$
shows 
$ |\widetilde{J}^{{\theta^\star}}_0(
     \nu)
     -{J}^{{\theta^\star}}(
     \phi) |
\le  C\|\lambda\|^2_\infty$.
By Lemma 
\ref{lemma:regularity},
the second term of \eqref{eq:decomposition_exploration} can be estimated by
$ {J}^{{\theta^\star}}(
     \phi)
     -  J^{{\theta^\star}}(
      \phi^{ \theta^\star})   
      \le C(\|K-K^{\theta^\star}\|^2_{\infty}+\|k-k^{\theta^\star}\|^2_{\infty})
$, which subsequently leads to \eqref{eq:gaussian_relax_optimal}.

Now
for each $m\in \sN$,
 consider  
 $$
    \widetilde{J}^{\theta^\star}_0(\nu^m) 
=
    \sE\Bigg[
    \int_0^T 
    \left(\int_\sR 
    f(t,X^{ \nu^m}_t,a)\nu^m(t, X^{ \nu^m}_t; \d a)
    \right)
    \,\d t
    +g(X^{\nu^m}_T)\bigg| \cF^{\textrm{ob}}_{m-1}
    \bigg],
  $$  
  where
  $\nu^m$ is defined  in Setting \ref{setting:randomised_algorithm},
  and 
   $X^{\nu^m}$ satisfies 
for all $t\in [0,T]$,
$$\d X^{\nu^m}_t = 
    \int_\sR (A^\star X^{\nu^m}_t+B^\star  a) \nu^m(t,X^{\nu^m}_t; \d a)\, \d t
    +  
    \bar{\sigma}_t\,
     \d W^m_t,\quad 
X^{\nu^m}_0 = x_0.$$
The assumptions in  Setting \ref{setting:randomised_algorithm}
ensures that $(K^m,k^m)_{m\in \sN\cup\{0\}}$ are
uniformly bounded. 
Then
by \eqref{eq:gaussian_relax_optimal},
there exists $C\ge 0$ such that 
for all $m\in \sN$,
\begin{align}
\label{eq:gaussian_relax_optimal_m}
|\widetilde{J}^{{\theta^\star}}_0(
     \nu^m)
     -
      J^{{\theta^\star}}(
      \phi^{ \theta^\star})|
      \le 
      C\left(
\|\lambda^m\|_\infty^2+
\|K^m-K^{\theta^\star}\|^2_{\infty}+\|k^m-k^{\theta^\star}\|^2_{\infty}
      \right),
\quad       \textnormal{$\d\sP$-a.s.}
\end{align}
Moreover, 
Theorem \ref{Thm:execution gap_drift_control}
and the uniform boundedness of $K^m,k^m, \lambda^m$ imply that 
$$
\sE \Big[
J^{{\theta^\star}}(\varphi^m)- \widetilde{J}^{\theta^\star}_0(\nu^m)
\Big|  \cF^{\textrm{ob}}_{m-1} \Big]
\leq C {|\pi_m|} 
 \|\lambda^m\|^2_\infty.
$$
which along with the tower property of conditional expectations and \eqref{eq:gaussian_relax_optimal_m} 
leads to the desired estimate.
 \end{proof}

\subsection{Regret analysis of Algorithm \ref{alg:exploration}}
\label{sec:analysis_exploration}

This section establishes the regret of  Algorithm \ref{alg:exploration} 
(i.e., Theorem \ref{thm:vanish_exploration}).
For notational simplicity, 
we denote by
 $C>0$  a  generic constant
which is
independent of the episode number  $m$ 
and may take   different values at each occurrence.
The following lemma proves that the policies 
generated by 
Algorithm 
\ref{alg:exploration} satisfies Setting   \ref{setting:randomised_algorithm},
and hence 
the results in 
 Section \ref{sec:general_policy}
 apply to these policies. 
The proof is given in Appendix \ref{appendix:techinical_lemma}.
 
\begin{Lemma}
\label{lemma:bounded_exploration}
Suppose (H.\ref{assum:lq_rl}) and (H.\ref{eq:bounded}) hold and $\sup_{m \in \sN} \varrho_m< \infty$. 
Let 
  $k^1,K^1\in C([0,T];\sR)$, $\lambda^1\in 
 C([0,T];\sR_{\ge 0})$
 and  $\nu^1\sim \cG(k^1,K^1,\lambda^1)$.
Then 
the policies 
$(\varphi^m)_{m\in\sN}$   from   Algorithm 
\ref{alg:exploration} satisfy Setting   \ref{setting:randomised_algorithm}.
\end{Lemma}


We first establish the convergence rate of $( {\theta}_m)_{m\in\sN}$ for suitable choices of $(\varrho_m)_{m\in\sN}$ and $(\pi_m)_{m\in\sN}$.

\begin{Proposition}
\label{Prop:concentration for theta explore}
Suppose (H.\ref{assum:lq_rl}) and (H.\ref{eq:bounded}) hold. 
Let $\theta_0\in \sR^{1\t 2}$, 
$V_0\in \sS_+^2$, 
$\varrho_0>0$
and 
$\nu^1\sim \cG(k^{\theta_0},K^{\theta_0}, \sqrt{\frac{\varrho_0}{2R}})$.
 Then there exist constants $c_0, C>0$
 such that if one sets 
$$
 \varrho_m=\varrho_0 m^{-1/2} \ln (m+1)
 \quad 
 \textnormal{and}\q 
|\pi_m| \varrho_{m-1} \leq c_0 {(\varrho_{m-1}   \wedge 1)},
\quad \fa  m\in \sN,
$$
then 
$\sE \big[|\theta_m - \theta^\star|^2 \big] \leq Cm^{-1/2} $
for all $m\in \sN$.
\end{Proposition} 

\begin{proof}

 We first  establish a lower and upper bound
of $(V^{-1}_m)_{m\in \sN}$ in terms of $m$.
Note that 
by \eqref{eq:optimal_relax_theta_rho} 
and the choice of $\nu^1$,
Algorithm  \ref{alg:exploration} defines  $K^{m}=K^{\theta_{m-1}}$, 
$k^{m}=k^{\theta_{m-1}}$ and 
$(\lambda^{m})^2=\tfrac{\varrho_{m-1}}{2R}$
for all $m\in \sN$.
 As $0< \inf_{t\in [0,T]}R_t\le \sup_{t\in [0,T]}R_t<\infty$, 
 by
 choosing a sufficiently small $c_0>0$
and  setting 
$ |\pi_m| \varrho_{m-1} \leq c_0 {(\varrho_{m-1}   \wedge 1)}$ for all $m\in \sN$,
Proposition \ref{prop:upper_lower_expectation}
Item \ref{eq:G_conditonal_lower} implies that 
for all $m\in \sN$,
\begin{align}
\label{eq:min_V_m_rho_n}
\sum_{n=1}^{m}\Lambda_{\min}\left(
\sE\left[
\int_0^T
\frac{1}{\bar{\sigma}^{2}_s} Z^{  n}_s (  Z^{ n}_s)^\top \,\d s
\big|
\cF^{\textrm{ob}}_{n-1}
 \right]\right)
 \ge r\sum_{n=1}^{m}  (\varrho_{n-1} \wedge 1) ,
\end{align}
with some constant  $ r >0$ independent of $m$. By the definition of $(\varrho_m)_{m\in\sN}$, there exists $\bar{N}_0 \in \sN$ such that
$\varrho_m \leq 1$ for all $m \geq \bar{N}_0$.  In particular, for $m \geq 4\bar{N}_0$, it follows from \eqref{eq:min_V_m_rho_n}
 that
\begin{align}
\begin{split}
\label{eq:min_V_m_rho_n_step2}
& \sum_{n=1}^{m}\Lambda_{\min}\left(
\sE\left[
\int_0^T
\frac{1}{\bar{\sigma}^{2}_s} Z^{  n}_s (  Z^{ n}_s)^\top \,\d s
\big|
\cF^{\textrm{ob}}_{n-1}
 \right]\right)  \geq r \sum_{n=\bar{N}_0  }^{m} \varrho_0(n-1)^{-1/2} \ln n  \ge     r  \int_{ \bar{N}_0 }^m  
  (x-1)^{-1/2} \ln x \, \d x
 \\
&\qquad \ge 
      r    \int_{ m/4}^m  
  (x-1)^{-1/2} \ln x \, \d x
  \ge  r  \ln\left(\tfrac{m}{4}\right) \int_{m/4}^m (x-1)^{-1/2} \d x \geq r m^{1/2} \ln(m+1).
\end{split}
\end{align}
By applying Proposition \ref{prop:concentration_G} with $\delta = 1/m$,  for each $m\in \sN$, with probability at least $1-1/m$, 
\begin{align}
\label{eq:concentration_V_inverse_delta_m}
\begin{split}
&
\left|
V^{-1}_{ m}-
V^{-1}_{ 0}-
\sum_{n=1}^m\sE\left[
\int_0^T
\frac{1}{\bar{\sigma}^{2}_t} Z^{  n}_t (  Z^{  n}_t)^\top \,\d t
\bigg|
\cF^{\textrm{ob}}_{n-1}
 \right]
 \right|\leq
 C \sqrt{ m  \ln (m+1)}.
\end{split}
\end{align}
Let $A_m$ be the event that \eqref{eq:concentration_V_inverse_delta_m} holds. 
Without loss of generality,   assume that $\bar{N}_0$ is sufficiently large such that  $C \ln^{-1/2}(m+1) \leq r/2$ for all $m \geq 4 \bar{N}_0$,
with the constant $C$ in \eqref{eq:concentration_V_inverse_delta_m}.
  Then, for any   $m \geq 4 \bar{N}_0$, 
by \eqref{eq:min_V_m_rho_n_step2} and 
\eqref{eq:concentration_V_inverse_delta_m},
under the event 
$A_m$,
\begin{align*}
&\Lambda_{\min}(V_m^{-1})
\\
&\ge \sum_{n=1}^{m} \Lambda_{\min}\left(
V^{-1}_{ 0}+
 \sE\left[
\int_0^T
\frac{1}{\bar{\sigma}^{2}_s} Z^{  n}_s (  Z^{ n}_s)^\top \,\d s
\big|
\cF^{\textrm{ob}}_{n-1}
 \right]\right)
 -\left|
V^{-1}_{ m}-V^{-1}_{0}-
\sum_{n=1}^m\sE\left[
\int_0^T
\frac{1}{\bar{\sigma}^{2}_t} Z^{  n}_t (  Z^{  n}_t)^\top \,\d t
\bigg|
\cF^{\textrm{ob}}_{n-1}
 \right]
 \right|\\
 & \geq  r m^{\frac{1}{2}} \ln(m+1) - \big(C \ln^{-\frac{1}{2}} (m+1) \big) m^{\frac{1}{2}} \ln(m+1) \geq  \tfrac{1}{2} r m^{\frac{1}{2}} \ln(m+1),
\end{align*}
which along with 
 $\big(\Lambda_{\min}(V_m^{-1})\big)^{-1} \leq \big(\Lambda_{\min}(V_0^{-1})\big)^{-1} $ 
 and $\sP(A_m^c)\le 1/m$
 shows that  for all $m \geq 4 \bar{N}_0$,
\begin{equation}
    \label{eq:expected value eigen classic}
\begin{aligned}
    \sE\Big[\big(\Lambda_{\min}(V_m^{-1})\big)^{-1} \Big] &=  \sE\Big[\big(\Lambda_{\min}(V_m^{-1})\big)^{-1} \bm{1}_{A_m}\Big] + \sE\Big[\big(\Lambda_{\min}(V_m^{-1})\big)^{-1} \bm{1}_{A^c_m}\Big] \\
    &\leq C  m^{-1/2} \big( \ln(m+1) \big)^{-1} + \frac{1}{m} \big(\Lambda_{\min}(V_0^{-1})\big)^{-1} \leq C m^{-1/2} \big( \ln(m+1) \big)^{-1}.
\end{aligned}
\end{equation}
Now  consider the event such that the estimate \eqref{eq: high prob parameter} in Lemma \ref{lemma:theta_accuracy} and \eqref{eq:concentration_V_inverse_delta_m} hold simultaneously. Then 
for all $m \geq 4 \bar{N}_0$,
with probability at least $1-2/m$, 
\begin{equation}
    \label{eq:projection bound}
     | \hat{ {\theta}}_{ m} - {\theta}^\star |^2 \leq C  m^{-1/2} \big( \ln(m+1) \big)^{-1} \big(1 + \ln m \big).
\end{equation}
As $\theta^\star \in
  \operatorname{int}(\Theta)
 $ by (H.\ref{eq:bounded}), 
 one can choose a sufficiently large $\bar{N}_0\in \sN$ such that 
  for all $m \geq 4 \bar{N}_0$, $C  m^{-1/2} \big( \ln(m+1) \big)^{-1}  \big(1 + \ln m \big) \leq \inf_{\theta \notin \Theta} |\theta - \theta^\star|^2$,
  with the constant $C$ in \eqref{eq:projection bound}.
  Hence, $\hat{ {\theta}}_{ m} \in \Theta$ on the event that \eqref{eq:projection bound} holds, i.e. $\sP(\hat{ {\theta}}_{ m} \notin \Theta) \leq 2/m$ for all $m \geq 4 \bar{N}_0$. Combining this with  the estimate \eqref{eq:expected parameter} in Lemma \ref{lemma:theta_accuracy}
  and \eqref{eq:expected value eigen classic}, for all $m \geq 4 \bar{N}_0$, 
 \begin{align*}
     \sE\big[ | { {\theta}}_{ m} - {\theta}^\star |^2 \big] &\leq C \left( C  m^{-1/2} \big( \ln(m+1) \big)^{-1}  (1 + \ln m) + \frac{1}{m} + \frac{2}{m}\right) \leq C m^{-1/2}.
 \end{align*}
By the uniform boundedness of   $(\theta_m)_{m\in \sN}\subset \Theta$, there exists  a sufficiently large
  $C$ 
  such that the estimate holds for all $m \in \sN$. 
 \end{proof}

%
\begin{proof}[Proof of Theorem \ref{thm:vanish_exploration}]

By Lemma \ref{lemma:regularity} Item  \ref{item:continuous_differentiability},
$\sR^{1\t 2} \ni \theta\mapsto  (K^\theta,k^\theta)\in C([0,T];\sR)^2$ is continuously differentiable and hence locally Lipschitz continuous. 
Then the fact that $(\theta_m)_{m\in \sN}\subset \Theta$ 
and  the boundedness of $\Theta$ (see (H.\ref{eq:bounded}))  
imply that for all $m\in \sN$ and $\sP$-a.s.,
\bb
\label{eq:K_theta_m_error}
\|K^{m}-K^{\theta^\star}\|^2_{\infty}
+\|k^{m}-k^{\theta^\star}\|^2_{\infty}
 = \|K^{\theta_{m-1}}-K^{\theta^\star}\|^2_{\infty}
 +\|k^{\theta_{m-1}}-k^{\theta^\star}\|^2_{\infty}
  \leq C|\theta_{m-1} - \theta^\star|^2.
  \ee
Hence, 
by Proposition \ref{prop: sensitivity} and the fact that $\|\lambda^m\|^2_\infty \leq C {\varrho_{m-1}}$ for all $m\in \sN$, \begin{align*}
\sE \big[ | J^{{\theta^\star}}(\varphi^m)-  J^{{\theta^\star}}(
      \phi^{ \theta^\star})| \big]
&\le
 C
\bigg({(1+|\pi_m|) \varrho_{m-1}} 
+
\sE \big[|\theta_{m-1} - \theta^\star|^2 \big]
\bigg), \qquad \forall m \in \sN.
\end{align*}
Combining this  with Proposition \ref{Prop:concentration for theta explore}
shows that 
  there exists  $c_0>0$
 such that if one sets 
\bb\label{eq:regret_exploration_condition_ref}
 \varrho_m= \varrho_0 m^{-1/2} \ln (m+1),
 \quad 
|\pi_m| \varrho_{m-1} \leq c_0 {(\varrho_{m-1}   \wedge 1)},
\quad \fa  m\in \sN,
\ee
then for all $N\in \sN$,
\begin{align*}
    \sE \big[ \textrm{Reg}(N) \big] &= \sum_{m=1}^N \sE \big[ | J^{{\theta^\star}}(\varphi^m)-  J^{{\theta^\star}}(
      \phi^{ \theta^\star})| \big] \leq C\sum_{m=1}^N \bigg({|\pi_m| \varrho_{m-1}} 
+
\varrho_{m-1}+
\sE \big[|\theta_{m-1} - \theta^\star|^2 \big] \bigg) \\
&\leq C\sum_{m=1}^N \bigg({|\pi_m| \varrho_{m-1}} 
+
\varrho_{m-1}+
m^{-1/2} \bigg).
\end{align*}
In particular, 
by choosing  $|\pi_m| \leq \frac{c_0}{\max(\sup_{m \in \sN}\varrho_{m-1}, 1)}$,
the condition \eqref{eq:regret_exploration_condition_ref} holds and
\begin{align*}
    \sE \big[ \textrm{Reg}(N) \big] 
&\leq C\sum_{m=1}^N \big( 
\varrho_{m-1}+
m^{-1/2} \big) \leq C N^{1/2} \ln (N+1).
\qedhere
\end{align*}
%
\end{proof}

\subsection{Regret analysis of Algorithm \ref{alg:MD}}
\label{sec:analysis_MD}

This section establishes the regret of  Algorithm \ref{alg:MD} 
(i.e., Theorem \ref{thm:MD}).
As in Section \ref{sec:analysis_exploration}, 
 $C>0$ denotes a  generic constant
which is
independent of the episode number  $m$ 
and may take   different values at each occurrence.
The following lemma proves 
an analogue of Lemma  \ref{lemma:bounded_exploration}
for
Algorithm 
\ref{alg:MD}.
The proof is given in Appendix \ref{appendix:techinical_lemma}.

\begin{Lemma}
\label{lemma:bounded_MD}
Suppose (H.\ref{assum:lq_rl}) and (H.\ref{eq:bounded}) hold. Let 
$k^1,K^1\in C([0,T];\sR)$, $\lambda^1\in 
 C([0,T];\sR_{> 0})$
 and  $\nu^1\sim \cG(k^1,K^1,\lambda^1)$.
Then 
the policies 
$(\varphi^m)_{m\in\sN}$   from   Algorithm 
\ref{alg:MD} satisfy Setting   \ref{setting:randomised_algorithm}.
\end{Lemma}


 The following lemma will be used to quantify the   dependence of
$\|K^{m}-K^{\theta^\star}\|_{\infty}$
and   
$\|k^{m}-k^{\theta^\star}\|_{\infty}$
on $(\theta_{n}-\theta^\star)_{n=0}^m$.
\begin{Lemma}
\label{lem:size of h^m}
Suppose (H.\ref{assum:lq_rl}) holds
and 
for each $m\in \sN$, let $h^m$ be  defined by \eqref{eq:MD_relaxed}.
Then 
$h^m_t = 1- \frac{(\lambda^{m+1}_t)^{2}}{(\lambda^{m}_t)^{2}}$
for all $m \in \sN$ and $t \in [0,T]$.
Assume further that 
$(\varrho_m)_{m\in \sN}$ is increasing, i.e., 
$\varrho_m\le \varrho_{m+1}$ for all $m\in \sN$.
Then 
$ h^{m}_t \leq \frac{1}{m} $ for all 
$m \in \sN$ and $t \in [0,T]$.
\end{Lemma}

\begin{proof}
By \eqref{eq:MD_relaxed},
    for all $m\in \sN$,

$$    1-h^m_t =
     \frac{{\varrho_m} (\lambda^{m}_t)^{-2}}{2R_t +   \varrho_m(\lambda^{m}_t)^{-2}} 
     =
        \frac{ \varrho_m (\lambda^{m}_t)^{-2}}{ \varrho_m  (\lambda^{m+1}_t)^{-2}}
        = \frac{(\lambda^{m+1}_t)^{2}}{(\lambda^{m}_t)^{2}}.
     $$
Assume further that  $(\varrho_m)_{m\in \sN}$ is increasing, by \eqref{eq:MD_relaxed}, 
for all $m\in \sN$,
$$\frac{\varrho_{m}}{2R_t}(\lambda^m_t)^{-2} \geq  \frac{\varrho_{m-1}}{2R_t} (\lambda^m_t)^{-2} = 1 + \frac{\varrho_{m-1}}{2R_t}(\lambda^{m-1}_t)^{-2} \geq \cdots \geq m-1 + \Big(\frac{\varrho_1}{2R_t}\Big)(\lambda^{1}_t)^{-2} \geq m-1,$$
which along with 
$
  h^m_t = \big( {1 + \frac{\varrho_m}{2R_t} (\lambda^{m}_t)^{-2}}\big)^{-1}
$ leads to the desired estimate. 
\end{proof}

Based on Lemma \ref{lem:size of h^m},
we quantify the accuracy of 
$K^{m}$
and   
$k^{m}$ in terms of $(\theta_n)_{n=0}^m$.

\begin{Proposition}
 \label{prop: action gap}
Suppose (H.\ref{assum:lq_rl}) and (H.\ref{eq:bounded}) hold. 
Let 
 $\nu^1\sim \cG(k^1,K^1,\lambda^1)$ 
for some  $k^1,K^1\in C([0,T];\sR)$  and $\lambda^1\in 
 C([0,T];\sR_{> 0})$.
Then  
 there exists a constant $C \geq 0$
 such that 
 for all 
 $\theta_0\in \sR^{1\t 2}$, 
$V_0\in \sS_+^2$,
 $m\in \sN$
 and   increasing sequences  
 $(\varrho_m)_{m\in \sN}\subset (0,\infty)$,
  \begin{align*}
&\|K^{m+1}-K^{\theta^\star}\|^2_{\infty}
+\|k^{m+1}-k^{\theta^\star}\|^2_{\infty}
\\
&\q 
\le C \sum_{n=1}^m 
      \frac{1}{n}\left\|\frac{\lambda^{m+1}}{ \lambda^{n+1}}\right\|_\infty^2| {\theta}_{n} - \theta^\star|^2 +
\left\|\frac{\lambda^{m+1} }{ \lambda^2}\right\|^2_\infty
\left(
\|K^{1}-K^{\theta^\star}\|^2_{\infty}
+\|k^{1}-k^{\theta^\star}\|^2_{\infty}
\right),
 \end{align*}
where for each $m\in \sN$,
 $K^{m+1}$, $k^{m+1}$ and $\lambda^m$ are   defined by \eqref{eq:MD_relaxed}.

\end{Proposition}

\begin{proof}
For all $m\in \sN$ and $t\in [0,T]$, 
by \eqref{eq:MD_relaxed}, 
$h_t^m\in [0,1]$ and 
\begin{align*}
|K^{m+1}_t-K^{\theta^\star}_t|\le 
h_t^m |K_t^{ \theta_m}-K^{\theta^\star}_t|+(1-h_t^m) 
|{K}^{m}_t-K^{\theta^\star}_t|.
\end{align*}
Then by using the convexity of  $\sR\ni x\mapsto x^2\in \sR$,
\begin{align*}
|K^{m+1}_t-K^{\theta^\star}_t|^2 \le 
h_t^m |K_t^{ \theta_m}-K^{\theta^\star}_t|^2 +(1-h_t^m) 
|{K}^{m}_t-K^{\theta^\star}_t|^2.
\end{align*}
Similar to \eqref{eq:K_theta_m_error},
one has 
$\|K^{ \theta_m}-K^{\theta^\star}\|^2_\infty\le C|\theta_m-\theta^\star|^2$,
which subsequently leads to 
\begin{equation*}
    |K^{m+1}_t-K^{\theta^\star}_t |^2
   \leq C {h}_t^m |{\theta}_{m} - \theta^\star|^2 + (1-{h}^m_t)
  |K^{m}_t-K^{\theta^\star}_t|^2.
\end{equation*}
Applying the same argument to 
$|k^{m+1}_t-k^{\theta^\star}_t |$
leads to the estimate:
\begin{equation}
    \label{eq:action gap time dependence}
    |K^{m+1}_t-K^{\theta^\star}_t |^2
    +|k^{m+1}_t-k^{\theta^\star}_t |^2
   \leq C {h}_t^m |{\theta}_{m} - \theta^\star|^2 + (1-{h}^m_t)
  (|K^{m}_t-K^{\theta^\star}_t|^2
  +|k^{m}_t-k^{\theta^\star}_t |^2).
\end{equation}

For each $m\in \sN$ and $t\in [0,T]$,
let 
    $\kappa_{m,t} \coloneqq
     |K^{m}_t-K^{\theta^\star}_t|^2
+|k^{m}_t-k^{\theta^\star}_t|^2$, 
$\Delta_m \coloneqq | {\theta}_{m} - \theta^\star|^2$ and $\eta_{m,t} \coloneqq  \prod_{k=1}^m(1-{h}^k_t) = \frac{(\lambda^{m+1}_t)^{2}}{(\lambda^{1}_t)^{2}}$ by Lemma \ref{lem:size of h^m}. 
Then for all $m\in \sN$ and $t\in [0,T]$,   \eqref{eq:action gap time dependence} and the fact that $h^{m+1}_t \leq \frac{1}{m}$ imply that
  \begin{align*}
      \eta_{m,t}^{-1}\kappa_{m+1,t} &\leq C  \tfrac{\eta_{m,t}^{-1}}{m}  \Delta_m + \eta_{m-1,t}^{-1}\kappa_{m-1,t} \leq C \tfrac{\eta_{m,t}^{-1}}{m} \Delta_m + C \tfrac{\eta_{m-1,t}^{-1}}{m-1}
      \Delta_{m-1} + \eta_{m-2,t}^{-1}\kappa_{m-2,t} \\
      &\leq C \sum_{n=1}^m 
      \tfrac{\eta_{n,t}^{-1}}{n}\Delta_n +
      \eta_1^{-1}\kappa_{1,t}.
  \end{align*}
Multiplying both sides by $\eta_m$
and using $\frac{\eta_{m,t}}{\eta_{n,t}}=\frac{(\lambda^{m+1}_t)^{2}}{(\lambda^{n+1}_t)^{2}}$ 
 yield
  \begin{align*}
\kappa_{m+1,t} 
      &\leq C \sum_{n=1}^m 
      \frac{1}{n}\left(\frac{\lambda^{m+1}_t }{ \lambda^{n+1}_t}\right)^2\Delta_n +
\left(\frac{\lambda^{m+1}_t }{ \lambda^2_t}\right)^2\kappa_{1,t}.
  \end{align*}
Taking the supremum over $t$ leads to the desired inequality. 
  \end{proof}
  
  To simplify the estimate in Proposition \ref{prop: action gap}, we quantify the behavior of 
   $(\lambda^m)_{m\in \sN}$
  in terms of $m$.
   \begin{Lemma}
\label{lem:size of lambda^m}
Suppose (H.\ref{assum:lq_rl}) holds. 
Let $\varrho_0>0$ and $\inf_{t \in [0,T]} \lambda^1_t > 0$.
 Then there exist constants $\ol{c}\ge  \ul{c}>0$
 such that if one sets   
$
 \varrho_m=\varrho_0 m^{1/2} \ln (m+1)$ for all $m\in\sN$, then 
$$\ul{c} m^{-1/2} \ln (m+1) \leq (\lambda^m_t)^2 \leq \ol{c} m^{-1/2} \ln (m+1),
\quad \fa t \in [0,T], m\in \sN\cap [2,\infty).
$$ 
where  
$\lambda^m$, for $m\in \sN\cap [2,\infty)$,    is   defined by \eqref{eq:MD_relaxed}.
\end{Lemma}

\begin{proof}

Observe from (H.\ref{assum:lq_rl})
that $\inf_{t\in [0,T]} R_t>0$. 
Then 
by \eqref{eq:MD_relaxed}
and the choice of $(\varrho_m)_{m\in \sN}$,
for all $m\ge 2$ and $t\in [0,T]$,
    \begin{align*}
      (\lambda^{m}_t)^{-2} &= {2R_t}{\varrho^{-1}_{m-1}} + (\lambda^{m-1}_t)^{-2} = 2R_t \sum_{k=1}^{m-1} \varrho_k^{-1} + (\lambda^1_t)^{-2} \geq C \sum_{k=1}^{m-1} k^{-1/2} \left( \ln (k+1) \right)^{-1}\\
     &\geq  C \left( \ln (m+1)   \right)^{-1} \sum_{k=1}^{m-1} k^{-1/2}
      \geq   C \left( \ln (m+1) \right)^{-1} \int_1^m x^{-1/2} \d x  \ge  C m^{1/2} \left( \ln (m+1) \right)^{-1}.
  \end{align*}
Rearranging the terms leads to the desired upper bound.

On the other hand, 
observe that there exists $C>0$ such that for all $x\ge 1$, 
$$\frac{\d }{\d x} x^{1/2} \big(\ln (x+2) \big)^{-1} =  \frac{(x+2)\ln (x+2)-2x}{2x^{1/2}(x+2)\ln ^2(x+2)}\geq C x^{-1/2} \big( \ln (x+1) \big)^{-1}.$$
Then 
by the facts that  $\sup_{t\in [0,T]} R_t<\infty$
and 
$(\lambda_t^1)^{-2}\le 
C
   $,
for all $m\ge 2$ and $t\in [0,T]$,
   \begin{align*}
      (\lambda^{m}_t)^{-2} & = 2R_t \sum_{k=1}^{m-1} \varrho_k^{-1} + (\lambda^1_t)^{-2} \leq C  \sum_{k=1}^{m-1} k^{-1/2} \left( \ln (k+1) \right)^{-1}
      \leq C \sum_{k=1}^{m-1} \frac{(k+2)\ln (k+2)-2k}{2k^{1/2}(k+2)\ln^2(k+2)}  
\\& \le      C  \int_1^m  \frac{(x+2)\ln(x+2)-2x}{2x^{1/2}(x+2)\ln^2(x+2)} \, \d x \leq
 C  m^{1/2} \left( \ln (m+2) \right)^{-1} 
 \leq  C m^{1/2} \left( \ln (m+1) \right)^{-1}.
  \end{align*}
  Rearranging the terms yields to the desired lower bound.
\end{proof}

Based on 
Lemma \ref{lem:size of lambda^m},
the next proposition  proves  the convergence rate of $( {\theta}_m)_{m\in\sN}$
in terms of $m$.

\begin{Proposition}
\label{Prop:concentration for theta MD}
Suppose (H.\ref{assum:lq_rl}) and (H.\ref{eq:bounded}) hold. 
Let $\theta_0\in \sR^{1\t 2}$, 
$V_0\in \sS_+^2$
and
$\varrho_0>0$.
 Then there exist constants $c_0, C>0$
 such that if one sets 
$\nu^1\sim \cG\left(k^{\theta_0},K^{\theta_0}, \varrho_0 
\right)$,
$
 \varrho_m=\varrho_0 m^{1/2} \ln (m+1)$ 
 and 
$|\pi_m|  \leq c_0$
for all $m\in \sN$,
then $\sE \big[|\theta_m - \theta^\star|^2 \big] \leq Cm^{-1/2}$ for all $m\in \sN$.
\end{Proposition} 

\begin{proof}

 By Lemma \ref{lem:size of lambda^m}, for all $m \geq 2$, $\|\lambda^m \|_\infty^2 \leq (\ol{c}/\underline{c})\min_{t\in [0,T]}(\lambda^m_t)^2$ and $\sup_{m \in \sN}  \|\lambda^m\|_\infty^{2} <\infty$. 
 Let $ c_0\coloneqq \tfrac{1}{C} \min\big(\underline{c}/\ol{c}, 
1/\left( \sup_{m \in \sN}  \|\lambda^m\|_\infty^{2}\right)\big)>0 $,
  with  $C$ being the constant in 
  Proposition \ref{prop:upper_lower_expectation}
Item \ref{eq:G_conditonal_lower}.
Hence, if  $|\pi_m|\le c_0$
 for all $m\in \sN$,
 then
$|\pi_{m}| \|\lambda^m\|_\infty^2\le \frac{1}{C}
\min_{t\in [0,T]}((\lambda^m_t)^2 \wedge 1)$ for all $m\in \sN$, 
and   
\begin{align}
\label{eq:min_V_m_rho_n_MD}
\sum_{n=1}^{m}\Lambda_{\min}\left(
\sE\left[
\int_0^T
\frac{1}{\bar{\sigma}^{2}_s} Z^{  n}_s (  Z^{ n}_s)^\top \,\d s
\big|
\cF^{\textrm{ob}}_{n-1}
 \right]\right)
 \ge r\sum_{n=1}^{m}  \min_{t\in [0,T]}((\lambda^n_t)^2 \wedge 1) ,
\end{align}
with some constant  $ r >0$ independent of $m$.
Moreover,  Lemma \ref{lem:size of lambda^m}
  implies that 
   there exists     $\bar{N}_0 \in \sN$ such that for all 
 $t \in [0,T]$ 
 and $m \geq \bar{N}_0$,
 $(\lambda^m_t)^2  \leq 1$ for all $t \in [0,T]$.
 Hence,   using \eqref{eq:min_V_m_rho_n_MD} and following the same argument as  in
Proposition \ref{Prop:concentration for theta explore}
yield (cf.~\eqref{eq:min_V_m_rho_n_step2}):
 \begin{align*}
\begin{split}
& \sum_{n=1}^{m}\Lambda_{\min}\left(
\sE\left[
\int_0^T
\frac{1}{\bar{\sigma}^{2}_s} Z^{  n}_s (  Z^{ n}_s)^\top \,\d s
\big|
\cF^{\textrm{ob}}_{n-1}
 \right]\right)
 \geq r m^{1/2} \ln(m+1),
 \q \fa m\ge 4\bar{N}_0.
\end{split}
\end{align*}
Now proceeding 
 along the lines of the proof of 
 Proposition \ref{Prop:concentration for theta explore}
 leads to the desired estimate. 
\end{proof}

%

\begin{proof}[Proof of Theorem \ref{thm:MD}]
Note that the choice of $(\varrho_m)_{m\in\sN}$ ensures that $(\varrho_m)_{m\in \sN}$ is increasing. 
By Propositions \ref{prop: action gap} and   \ref{Prop:concentration for theta MD},
for all $m\in \sN$,
 \begin{align*}
&\sE \big[ \|K^{m+1}-K^{\theta^\star}\|^2_{\infty} \big]
+\sE \big[ \|k^{m+1}-k^{\theta^\star}\|^2_{\infty} \big]
\\
&\q 
\le C \sum_{n=1}^m 
      \frac{1}{n}\left\|\frac{\lambda^{m+1} }{ \lambda^{n+1}}\right\|_\infty^2 \sE \big[ | {\theta}_{n} - \theta^\star|^2 \big] +
\left\|\frac{\lambda^{m+1} }{ \lambda^2}\right\|^2_\infty
\left(
\|K^{\theta_0}-K^{\theta^\star}\|^2_{\infty}
+\|k^{\theta_0}-k^{\theta^\star}\|^2_{\infty}
\right) \\
&\q 
\le  C \sum_{n=1}^m 
      \frac{1}{n^{3/2}}\left\|\frac{\lambda^{m+1} }{ \lambda^{n+1}}\right\|_\infty^2  +
C \left\|\frac{\lambda^{m+1} }{ \lambda^2}\right\|^2_\infty.
 \end{align*}
By Lemma \ref{lem:size of lambda^m},
for all $m\in\sN$,
 \begin{align*}
\sE \big[ \|K^{m+1}-K^{\theta^\star}\|^2_{\infty} \big]
+\sE \big[ \|k^{m+1}-k^{\theta^\star}\|^2_{\infty} \big]
&\le  C m^{-\frac{1}{2}} \ln(m+1) \left(   \sum_{n=1}^m 
      \frac{1}{n \ln (n+1)}  +
1
 \right) \\
 &\le  C m^{-\frac{1}{2}} \ln(m+1) \left(\ln \big( \ln(m+1) \big) +1\right).
 \end{align*}
Summing over the index $m$ then yields 
that 
for all $N\in \sN$,
 \begin{align}
 \label{eq:K_error_MD}
     \sum_{m=1}^N \left( \sE \big[ \|K^{m+1}-K^{\theta^\star}\|^2_{\infty} \big]
+\sE \big[ \|k^{m+1}-k^{\theta^\star}\|^2_{\infty} \big] \right) \leq CN^{\frac{1}{2}} \ln(N+1)
 \left(\ln \big( \ln(N+1) \big) +1\right).
 \end{align}
Note that 
the choices of $\varrho_m$ and $\pi_m$
and Lemma \ref{lem:size of lambda^m}
 imply that
$
\|\lambda^m\|_\infty^2 + {|\pi_m|} 
 \|\lambda^m\|_\infty^2
\le Cm^{-\frac{1}{2}}
\ln (m+1) 
,
$ for all $ m\in \sN\cap [2,\infty)$,
which along with 
\eqref{eq:K_error_MD} and  Proposition \ref{prop: sensitivity}
shows that for all $N \in \sN$,
\begin{align*}
       \sE \big[  \textrm{Reg}(N) \big] 
	&=
     \sE \big[ J^{{\theta^\star}}(\varphi^1)-  J^{{\theta^\star}}(
      \phi^{ \theta^\star}) \big]
      +\sum_{m=2}^N \sE \big[ J^{{\theta^\star}}(\varphi^m)-  J^{{\theta^\star}}(
      \phi^{ \theta^\star}) \big]
     \\
     &\leq C \bigg(
      \sum_{m=1}^N m^{-\frac{1}{2}}
\ln (m+1) 
+ N^{\frac{1}{2}} \ln(N+1)
 \left(\ln \big( \ln(N+1) \big) +1\right)  
 \bigg) 
 \\
 &
\leq C
N^{\frac{1}{2}} \ln(N+1)
 \left(\ln \big( \ln(N+1) \big) +1\right).
\end{align*}
This proves the desired estimate. 
\end{proof}

\appendix

  \begin{appendices}

\section{Proofs of technical lemmas}
\label{appendix:techinical_lemma}

\begin{proof}[Proof of Lemma \ref{lem:Weak Law for average Z}]
Throughout this proof, 
let $C$ be a generic constant independent of $N$,
$(\rho_i)_{i=1}^N$,
$(\beta_{ij})_{i,j=1}^N$ and $\delta$.
Item \ref{item:Weak Law for average Z}
follows directly from 
 the
general Hoeffding inequality
\cite[Theorem 2.6.3]{vershynin2018high}.
To prove Item \ref{item:Weak Law for sq},
observe that $(\zeta^2_i-1)_{i=1}^N$ are mean-zero sub-exponential random variables. Hence by
Bernstein's inequality 
in \cite[Theorem 2.8.2]{vershynin2018high},
there exists $C\ge 0$
such that 
for all $\delta\ge  0$,
it holds with probability at least $1-\delta$ that
\begin{align*}
\left|\sum_{i=1}^N \rho_i(\zeta^2_i-1)
\right|
&\le 
C\max\left(\|\rho\|_2
\sqrt{\ln \left(\frac{2}{\delta}\right)},
\|\rho\|_\infty\ln \left (\frac{2}{\delta}\right)
\right)
\le C\|\rho\|_2
\left(1+
\ln \left (\frac{1}{\delta}\right)
\right),
\end{align*}
where the last inequality used 
 the fact that $\|\rho\|_\infty\le \|\rho\|_2$
 and 
 the Cauchy-Schwarz inequality.

It remains to prove Item \ref{item:Weak Law for cross}.
The independence 
of $(\zeta_i)_{i\in \sN}$
implies that 
$\sE[\sum_{i,j=1,
i\not =j}^N \beta_{ij}\zeta_i\zeta_j]=0$.
Then, by considering
${A}\in 
\sR^{N\t N}$
with diagonal entries ${a}_{ii}=0$
for all $i=1,\ldots, N$,
and off diagonal entries 
${a}_{ij}=\rho^N_{ij}$
for all $i\not =j$,
and by applying the Hanson--Wright inequality
(see 
\cite[Theorem 6.2.1]{vershynin2018high}),
there exists a constant $C\ge 0$ such that 
 for all $s\ge 0$,
$$
\sP\left(
\left|\sum_{i,j=1,i\not =j}^N \beta_{ij}\zeta_i\zeta_j
\right|\ge s 
\right)
\le 
2\exp
\left(
-C\min
\left(
\frac{s^2}{\|{A}\|_F^2},
\frac{s}{\|{A}\|_F}
\right)
\right),
$$
where $\|{A}\|_F$ is the Frobenius norm of ${A}$.
Consequently for all 
$\delta>0$, it holds with probability at least $1-\delta$ that
\begin{align*}
\left|\sum_{i,j=1,i\not =j}^N \beta_{ij}\zeta_i\zeta_j
\right|
&\le 
C\|{A}\|_F
\max\left(\sqrt{\ln \left(\frac{2}{\delta}\right)},
\ln \left (\frac{2}{\delta}\right)
\right)
\le 
C\|{A}\|_F
\left(1+
\ln \left (\frac{1}{\delta}\right)
\right),
\end{align*}
where the last inequality used 
the Cauchy-Schwarz inequality.
\end{proof}

\begin{proof}[Proof of Lemma \ref{lemma:regularity}]
Item \ref{item:continuous_differentiability} has been proved in \cite[Lemma 3.1]{basei2022logarithmic} for $S=p=q=M=m=0$. The argument there is based on 
the implicit function theorem and extends straightforwardly  to the present setting with general convex quadratic cost functions (details omitted).

For Item \ref{item:quadratic_dependence},  
let $C$  be a generic constant 
 depending only on $L$ and the  coefficients in (H.\ref{assum:lq_rl}),
 and define $\|X\|_{\cS^2}=\sE[\sup_{t\in [0,T]}|X_t|^2]^{1/2}$ for any process $X:\Om\t [0,T]\to \sR$.
Consider the control processes 
 $U^{\phi^\star} =\phi^{\theta^\star}(\cdot,X^{\theta^\star, \phi^{\theta^\star}}_\cdot)$ 
 and 
  $U^{\phi } =\phi (\cdot,X^{\theta^\star, \phi}_\cdot)$ 
 $\d\sP\otimes \d t$ a.e.
As $\phi^{\theta^\star}$ is an optimal policy, 
a quadratic expansion of \eqref{eq:LQ_RL_cost}
at $U^{\phi^\star} $  
(see \cite[Proposition 3.7]{basei2022logarithmic})
implies that
\begin{align}
\label{eq:quadratic_expansion}
\begin{split}
 J^{  \theta^\star}(\phi)-     J^{  \theta^\star}( \phi^{ \theta^\star}) 
&=2\sE\left[\int_0^T 
\begin{pmatrix}
X^{\theta^\star, \phi}_t- X^{\theta^\star, \phi^{\theta^\star}}_t
\\
U^{\phi}_t- U^{\phi^\star}_t
\end{pmatrix}^\top
\begin{pmatrix}
Q_t & S_t
\\
S_t  & R_t
\end{pmatrix}
\begin{pmatrix}
X^{\theta^\star, \phi}_t- X^{\theta^\star, \phi^{\theta^\star}}_t
\\
U^{\phi}_t- U^{\phi^\star}_t
\end{pmatrix}
\,\d  t\right]
\\
&\quad +2 \sE\left[
M(X^{\theta^\star, \phi}_T- X^{\theta^\star, \phi^{\theta^\star}}_T)^2
\right]
\\
&\le C
\left(
\|X^{\theta^\star, \phi}- X^{\theta^\star, \phi^{\theta^\star}}\|^2_{\cS^2}
+
\|\phi^{\theta}(\cdot,X^{\theta^\star, \phi}_\cdot)- 
\phi^{\theta^\star}(\cdot, X^{\theta^\star, \phi^{\theta^\star}}_\cdot)\|^2_{\cS^2}
\right).
\end{split}
\end{align}
 Recall that $X^{\theta^\star, \phi}$ satisfies 
 \bb\label{eq:X_phi}
  \d X^{\theta^\star, \phi}_t = 
    (A^\star X^{\theta^\star, \phi}_t +B^\star (K_t X^{\theta^\star, \phi}_t+k_t))\, \d t
     +  
    \bar{\sigma}_t\,
     \d W_t,
     \quad  t\in [0,T]; \quad X^{\theta^\star, \phi}_0 = x_0.
 \ee
The analytical solution to \eqref{eq:X_phi}
and 
the assumption that $\|K\|_{\infty},\|k\|_{\infty}\le L$ yield
$\|X^{\theta^\star, \phi}\|_{\cS^2}\le C$. 
Observe 
that
$X^{\theta^\star, \phi}_0= X^{\theta^\star, \phi^{\theta^\star}}_0$ and for all $t\in [0,T]$,
\begin{align*}
\d
(X^{\theta^\star, \phi}_t- X^{\theta^\star, \phi^{\theta^\star}}_t)
&=
\left(
A^\star \big(X^{\theta^\star, \phi}_t - X^{\theta^\star, \phi^{\theta^\star}}_t\big)
+B^\star\big(
K_t X^{\theta^\star, \phi}+k_t
-K^{\theta^\star}_tX^{\theta^\star, \phi^{\theta^\star}}_t
-k^{\theta^\star}_t
\big)
\right) \d t 
\\
&=
\left(
(A^\star +B^\star K^{\theta^\star}_t)\big(X^{\theta^\star, \phi}_t - X^{\theta^\star, \phi^{\theta^\star}}_t\big)
+B^\star
(K_t- K^{\theta^\star}_t)X^{\theta^\star, \phi}_t+
B^\star(k_t
-k^{\theta^\star}_t)
\right)\d t. 
\end{align*}
Let  
$\Phi\in C([0,T]; \sR)$ be the fundamental  
solution of  
$\d \Phi^{\theta}_t=(A^\star+B^\star K^{\theta^\star}_t)\Phi^{\theta}_t \,\d t$.
Then 
\begin{align*}
X^{\theta^\star, \phi}_t- X^{\theta^\star, \phi^{\theta^\star}}_t
&=
\Phi^{\theta}_t  \int_0^t(\Phi^{\theta}_s)^{-1}
\left(
B^\star
(K_t- K^{\theta^\star}_t)X^{\theta^\star, \phi}_t+
B^\star(k_t
-k^{\theta^\star}_t)
\right)
\,\d s, \quad t\in [0,T],
\end{align*}
which along with $\|K^{\theta^\star}\|_{\infty}<\infty$ and $\|X^{\theta^\star, \phi}\|_{\cS^2}\le C$
shows 
$\|X^{\theta^\star, \phi}- X^{\theta^\star, \phi^{\theta^\star}}\|_{\cS^2}\le 
C(\|K-K^{\theta^\star}\|_{\infty}+\|k-k^{\theta^\star}\|_{\infty})$.
Furthermore, for all $t\in [0,T]$, 
\begin{align*}
\phi^{\theta}(t,X^{\theta^\star, \phi}_t)- 
\phi^{\theta^\star}(t, X^{\theta^\star, \phi^{\theta^\star}}_t)
&=K_tX^{\theta^\star, \phi}_t+k_t
-K^{\theta^\star}_tX^{\theta^\star, \phi^{\theta^\star}}_t
-k^{\theta^\star}_t
\\
&=
K^{\theta^\star}_t(X^{\theta^\star, \phi}_t - X^{\theta^\star, \phi^{\theta^\star}}_t)
+(K_t-K^{\theta^\star}_t)X^{\theta^\star, \phi}_t
+k_t -k^{\theta^\star}_t.
\end{align*}
Hence by $\|K^{\theta^\star}\|_{\infty}<\infty$ and  $\|X^{\theta^\star, \phi}\|_{\cS^2}\le C$,
$\|\phi^{\theta}(\cdot,X^{\theta^\star, \phi}_\cdot)- 
\phi^{\theta^\star}(\cdot, X^{\theta^\star, \phi^{\theta^\star}}_\cdot)\|_{\cS^2}
\le C(\|K-K^{\theta^\star}\|_{\infty}+\|k-k^{\theta^\star}\|_{\infty})
$, which along with 
\eqref{eq:quadratic_expansion} yields the desired estimate. 
\end{proof}

\begin{proof}
[Proof of Lemma \ref{lemma:bounded_exploration}]
Recall that for all $m\in \sN$,
 Algorithm 
\ref{alg:exploration} defines 
$K^{m+1}=K^{\theta_m}$
and
$k^{m+1}=k^{\theta_m}$.
 Lemma  \ref{lemma:regularity}
Item \ref{item:continuous_differentiability}
and 
(H.\ref{eq:bounded}) imply that 
$
\sup_{\theta\in \Theta}\left(
\|K^\theta\|_{\infty}+\|k^\theta\|_{\infty}\right)<\infty
$.
By the definition \eqref{eq:theta_m_project}, 
$(\theta_m)_{m\in\sN}$ takes values in a bounded set $\Theta$,
and hence 
$\sup_{ m\in \sN}( \|K^m\|_{L^\infty}+
\|k^m\|_{L^\infty})\le L$.
The uniform boundedness of $\|\lambda^m\|_\infty$ follows from  \eqref{eq:optimal_relax_theta_rho},
$\inf_{t\in [0,T]} R_t>0$, and 
 the condition that $\sup_{m \in \sN} \varrho_m  < \infty$.
Finally, 
 for each $m\in \sN$,
by
\eqref{eq: statistics} and \eqref{eq:theta_m_project},
 $\theta_m$ is  
$\cF^{\textrm{ob}}_{m}$ measurable,
which along with the continuity of 
 $\sR^{1\t 2}\t [0,T]\ni (\theta,t)\mapsto (K^\theta_t ,k^{\theta}_t)\in \sR^2$  leads to the desired  $\cF^{\textrm{ob}}_{m}\otimes \cB([0,T])$-measurability 
of $K^{m+1}$ and $k^{m+1}$.  
\end{proof}

\begin{proof}[Proof of Lemma \ref{lemma:bounded_MD}]
 By similar arguments as those for 
 Lemma \ref{lemma:bounded_exploration},
  $L\coloneqq \sup_{ m\in \sN}
  (\|K^{\theta_m}\|_{L^\infty}
  +\|k^{\theta_m}\|_{L^\infty})<\infty  $, 
and for all  $m\in \sN$,
 $K^{\theta_m}$ and $k^{\theta_m}$ are 
$\cF^{\textrm{ob}}_{m-1}\otimes \cB([0,T])$-measurable.
Observe that  for each $m\in \sN$ and $t\in [0,T]$,
$K^{m+1}_t$ (resp.~$k^{m+1}_t$) is a convex combination of $(K^{\theta_n}_t)_{n=0}^m$
(resp.~$(k^{\theta_n}_t)_{n=0}^m$) according to the weights 
$(h^n_t)_{n=0}^m$.
Hence 
 $K^{m}$ and $k^{ m}$
 are 
$\cF^{\textrm{ob}}_{m-1}\otimes \cB([0,T])$-measurable,
and
$\sup_{ m\in \sN}
  (\|K^{m+1}\|_{L^\infty}
  +\|k^{m+1}\|_{L^\infty})\le L$,
independent of  $\theta_0,V_0$, 
$(\varrho_m)_{m\in \sN}$  
and  $(\pi_m)_{m\in\sN}$. 
Finally, 
by \eqref{eq:MD_relaxed}, $\varrho_m >0$ and $R_t \geq 0$   for all $t \in [0,T]$,
we have 
$\lambda^{m+1}_t\le \lambda^m_t$ for all $t\in [0,T]$,
which along with 
the fact that $\|\lambda^0\|_\infty<\infty$
implies 
the uniform boundedness of $(\lambda^m)_{m\in \sN}$.
  \end{proof}
\end{appendices}

\bibliographystyle{siam}

\bibliography{relaxed_ctrl_lq_sicon.bib}

 \end{document}